\documentclass[11pt,a4paper]{article}
\usepackage{tabularx}
\usepackage[toc,page,header]{appendix}
\usepackage{minitoc}
\usepackage{array}
\usepackage{titletoc}
\usepackage{times,latexsym}
\usepackage{url}
\usepackage[T1]{fontenc}
\usepackage{etoc}
\usepackage{graphicx}
\usepackage{amsthm}  
\usepackage{amsmath} 
\usepackage{amsfonts}
\usepackage{booktabs}
\usepackage{adjustbox}
\usepackage{multirow}
\usepackage{adjustbox}
\usepackage{graphicx}
\usepackage{tcolorbox}
\usepackage{fbox}
\usepackage{amsthm}
\usepackage[normalem]{ulem}
\usepackage{hyperref}
\usepackage{url}
\usepackage{pifont}
\usepackage[ruled,vlined]{algorithm2e}
\usepackage{xcolor}
\usepackage{soul}
\usepackage{makecell,array}
\newtheorem{theorem}{Theorem}
\usepackage{cuted}    
\usepackage{caption}
\usepackage{arydshln}
   \usepackage[acceptedWithA]{tacl2021v1}
\usepackage{tacl2021v1}

\usepackage{xspace,mfirstuc,tabulary}

\usepackage[normalem]{ulem}
\useunder{\uline}{\ul}{}
\usepackage[normalem]{ulem}
\useunder{\uline}{\ul}{}

\newif\iftaclinstructions
\taclinstructionsfalse 
\iftaclinstructions
\renewcommand{\confidential}{}
\renewcommand{\anonsubtext}{(No author info supplied here, for consistency with
TACL-submission anonymization requirements)}
\newcommand{\instr}
\fi

\iftaclpubformat 

\else

\fi

\title{CAuSE: Decoding Multimodal Classifiers using Faithful Natural Language Explanation}

\newcommand\tab[1][0.7cm]{\hspace*{#1}}
\newcommand{\affa}{{$^{1}$}}
\newcommand{\affb}{{$^{2}$}}

\author{
  {\bf Dibyanayan Bandyopadhyay}\affa \tab 
  {\bf Soham Bhattacharjee}\affa \tab
  {\bf Mohammed Hasanuzzaman}\affb \\
  {\bf Asif Ekbal}\affa \\
  \affa Indian Institute of Technology Patna \tab 
  \affb Queen's University Belfast \\
  \affa \texttt{\{dibyanayan, sohambhattacharjeenghss, asif.ekbal\}@gmail.com} \\
  \affb \texttt{m.hasanuzzaman@qub.ac.uk}
}

\begin{document}
\maketitle
\begin{abstract}

     Multimodal classifiers function as opaque black box models. While several techniques exist to interpret their predictions, very few of them are as intuitive and accessible as natural language explanations (NLEs). To build trust, such explanations must faithfully capture the classifier’s internal decision making behavior, a property known as \emph{faithfulness}. In this paper, we propose \emph{CAuSE} (Causal Abstraction under Simulated Explanations), a novel framework to generate faithful NLEs for any pretrained multimodal classifier. We demonstrate that \emph{CAuSE} generalizes across datasets and models through extensive empirical evaluations. Theoretically, we show that \emph{CAuSE}, trained via interchange intervention, forms a causal abstraction of the underlying classifier. We further validate this through a redesigned metric for measuring causal faithfulness in multimodal settings.  \emph{CAuSE} surpasses other methods on this metric, with qualitative analysis reinforcing its advantages. We perform detailed error analysis to pinpoint the failure cases of \emph{CAuSE}. For replicability, we make the codes 
     available at \url{https://github.com/newcodevelop/CAuSE}. 

\end{abstract}

\section{Introduction}
Multimodal classifiers (e.g. VisualBERT~\cite{li2019visualbert}) integrate information from multiple modalities, such as images, text, and audio, and classify input into a predefined set of classes. These models are vital in a range of applications. For instance, given radiology reports in free-text format and corresponding chest X-ray images, a multimodal classifier can be trained to predict whether a patient has COVID-19~\citep{baltrušaitis2017multimodalmachinelearningsurvey}.

However, their widespread adoption depends on whether we can trust their predictions. This requires the development of interpretability techniques that explain \textit{how} the classifier came to its prediction. Input attribution methods aim to identify features or concepts in the input that influence the classifier’s decision. Although these techniques provide useful insights, they face two significant limitations: (i) \emph{Lack of natural language explanations (NLEs)}. Their outputs are typically low-level and not conveyed in natural language, which hampers interpretability~\citep{sundararajan2017axiomaticattributiondeepnetworks}; and (ii) \emph{Causal faithfulness}. They often do not capture a true causal link between the input and the prediction of the model~\citep{bandyopadhyay-etal-2024-seeing,chattopadhyay2019neural}. Crucially, faithfulness is essential to trust the predictions of a model.

To address these limitations, we propose \textbf{CAuSE (Causal Abstraction under Simulated Explanations)}, a novel post-hoc framework for generating \emph{causally faithful} NLEs for any \emph{frozen} multimodal classifier. CAuSE specifically targets \emph{discriminative} classifiers with an encoder-classification head architecture, which are found to have been deployed in today's production systems \citep{megahed2025adapting,JI2025102794}. Unlike generative models, these discriminative classifiers lack native explanation generation capabilities, necessitating post-hoc interpretability frameworks like CAuSE. Figure~\ref{fig:basic} shows an example of CAuSE explaining a \emph{frozen} multimodal offensiveness classifier decision for a meme.

At the core of CAuSE is a pretrained language model $\phi$, which is guided by the hidden states of the classifier to produce natural language explanations. CAuSE uses a novel loss function, grounded in interchange intervention training~\citep{DBLP:journals/corr/abs-2106-02997}, that enforces causal faithfulness of the generated explanations (see Section~\ref{res-all} for results
and analysis). We also propose a variant of a widely used causal faithfulness metric~\citep{atanasova-etal-2023-faithfulness}, termed \textbf{CCMR (Counterfactual Consistency via Multimodal Representation)}, tailored for evaluating faithfulness of NLEs in multimodal contexts. Under this metric, CAuSE demonstrates strong performance on benchmark datasets such as e-SNLI-VE~\citep{do2021esnlivecorrectedvisualtextualentailment}, which is a dataset of image premises and text hypotheses labeled with entailment, contradiction, or neutral labels; Facebook Hateful Memes~\citep{kiela2021hateful} in which the task is to predict whether an input meme is offensive or not; and VQA-X \cite{Park_2018_CVPR}, which is a visual question answering dataset coupled with gold explanations.

We conduct extensive qualitative analyses to examine: (i) where and how CAuSE succeeds in producing causally faithful NLEs (\S \ref{llm-as-judge} and \S \ref{quals-disc}), (ii) typical failure cases (\S \ref{failure-mode}), and (iii) general trends observed in error analysis (\S \ref{error-analysis}).

Our contributions are \textit{three-fold}: (1) a framework for generating faithful, post-hoc NLEs for multimodal classifiers; (2) a novel loss function that enforces causal faithfulness; and (3) an extensive empirical evaluation demonstrating the effectiveness of CAuSE in producing causally faithful NLEs.

\subsection{Scope and Applicability}

CAuSE serves dual purposes as both a deployable tool and a methodological framework with broader implications. 

\textbf{Immediate application.} As explained in the Introduction, CAuSE targets discriminative multimodal classifiers, which, unlike generative models, lack native explanation generation capabilities. Simulating these classifiers with separate generative models often produces plausible but unfaithful explanations that fail to reflect actual model reasoning \cite{madsen2024selfexplanations, turpin2023language}. CAuSE addresses this through faithful explanation generation, demonstrated across various classifiers studied later (\S Section~\ref{res-all}).

\textbf{Broader implications.} Beyond practical deployment, CAuSE demonstrates how causal abstraction principles that are typically validated on simpler networks can be scaled to complex transformer-based architectures with an application to faithful explanation generation that is empirically shown to be both task-agnostic and largely architecture-agnostic. This positions CAuSE as a blueprint for future explainability research requiring causal faithfulness.

\begin{figure}
    \centering
    \includegraphics[width=\columnwidth]{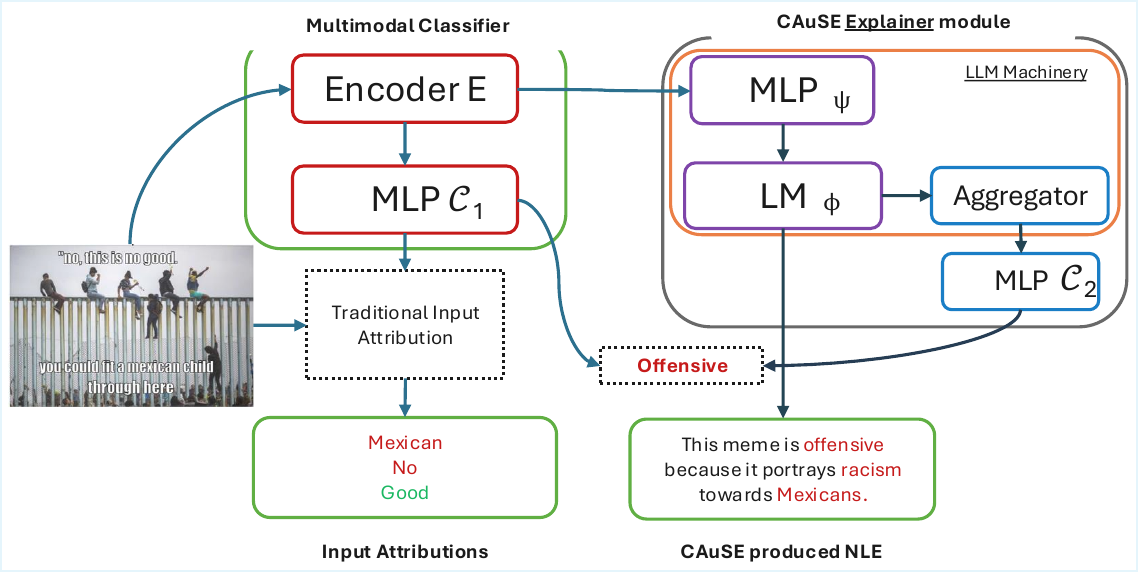}
    \caption{This figure shows the abstract schematic of CAuSE and how it explains \emph{discriminative} multimodal classifiers at inference. The internal components of CAuSE are i) an MLP $\psi$, ii) A language model (LM) $\phi$, and iii) an aggregator followed by a classifier $\mathcal{C}_2$. \emph{The input to the multimodal encoder is a meme which is composed of both text and image (separately not shown).}}
    \label{fig:basic}
\end{figure}

\section{Causal Abstraction and Interchange Intervention}\label{caii}

\subsection{Causal Abstraction}\label{ca}
In \citet{DBLP:journals/corr/abs-2106-02997}, the authors introduced the concept of causal abstraction for neural models. They define a neural network, $N_1$, as a causal abstraction of a higher-level causal model, $N_2$, if the neural representations of $N_1$ exhibit the same causal outcomes as the corresponding high-level variables in $N_2$. This alignment is achieved through the Interchange Intervention Training (IIT) objective.

This type of alignment learned through IIT is referred to as \textit{causal abstraction} in the literature. IIT ensures a systematic correspondence between interventions on the neurons in $N_1$ and the mapped variables in $N_2$. Unlike a traditional teacher-student training objective, which merely teaches the student to mimic the teacher's output, causal abstraction ensures that the student model internally mirrors the teacher's decision-making process. In particular, IIT guarantees that interventions on variables in $N_2$ yield analogous effects when applied to the associated neurons in $N_1$, thereby aligning the causal structure of the two models.

\subsection{Interchange Intervention}\label{ii}

To perform an interchange intervention, we begin by passing a "base" input $b$ through the network $N_1$, randomly selecting a neuron $i_1 \in N_1$ and freezing its activation at the value $i_1(b)$. Simultaneously, a second input, referred to as the "source" input $s$, is also passed through $N_1$, but with the activation of neuron $i_1$ held fixed at $i_1(b)$ instead of its native activation under $s$. The resulting output from this modified forward pass is denoted by $y_{i_1}^{INT}(N_1)$.

Assuming a one-to-one correspondence between each neuron $i_1$ in $N_1$ and a subset of variables $S(i_1)$ in a higher-level model $N_2$, we say that $N_2$ is a causal abstraction of $N_1$ if the following condition holds:
$$S(i_1) = \{i_2 \in N_2: [y_{i_2}^{INT} (N_2) = y_{i_1}^{INT} (N_1)]\}$$
In other words, after intervention in $i_1$ and $i_2$, output of $N_1$ must match that of $N_2$, under the same interchange intervention.

Let $P_{y_{i_1}^{INT}(N_1)}$ and $P_{y_{i_2}^{INT}(N_2)}$ denote the probability distributions over the model outputs under interchange intervention for $N_1$ and $N_2$, respectively. Following \citet{DBLP:journals/corr/abs-2106-02997}, causal abstraction is encouraged by minimizing the loss of interchange intervention training, defined as the Kullback–Leibler divergence between these two distributions.
\begin{equation}\label{iitloss}
\mathcal{L}_{IIT} = D_{KL}(P_{y_{i_1}^{INT}(N_1)}| P_{y_{i_2}^{INT}(N_2)})
\end{equation}
where $D_{KL}(P|Q)$ denotes the KL divergence between distributions $P$ and $Q$.

\section{Methodology}

\subsection{Setup and Objective}
The goal of CAuSE is to provide post hoc natural language explanations for the predictions of a pretrained multimodal classifier $M$.

Any multimodal classifier $M$ is composed of an encoder $E$ and a multilayered perceptron (MLP) $\mathcal{C}_1$. \emph{This model is pretrained and fixed; our method does not alter or retrain it.}

Given an input pair $(t, v)$ of text and image, $E$ computes a joint representation $c = E(t,v) \in \mathbb{R}^{m \times 1}$, which is then passed to $\mathcal{C}_1$ to obtain a logit vector $z \in \mathbb{R}^{L}$ over $L$ classes. The output prediction is $y_1 = \text{softmax}(z)$.

\subsection{Explainer Module (CAuSE Framework)}
\emph{CAuSE} is a framework consisting of an explanation generation module, referred to as the \textbf{Explainer}, along with a set of loss functions used to train it. The Explainer is composed of four key components:
(i) a multi-layer perceptron (MLP) $\psi$;
(ii) a language model $\phi$ responsible for generating natural language explanations (NLEs);
(iii) an aggregator module $\mathcal{A}$, which processes and transforms the token-level outputs of $\phi$ into a fixed-length feature representation; and
(iv) a classifier $\mathcal{C}_2$, which is structurally identical to the original classifier $\mathcal{C}_1$ and is trained to replicate its predicted label $y_1$.

\paragraph{Training Objective for $\phi$}

To generate faithful explanations grounded in the classifier $M$'s internal reasoning, we condition the language model $\phi$ on its hidden representation $c = E(z)$. Since $c$ may not align with $\phi$’s input embedding space, we project it using a simple MLP $\psi$, and use $\psi(c)$ as the embedding for the initial \textbf{BOS} token. The model $\phi$ is then fine-tuned using the standard causal language modeling objective:
\begin{equation}\label{clm}
\mathcal{L}_{\phi} = -\sum_{i=1}^{T} \log P_{\phi}(x_i | x_{i-1}, \ldots, x_0)
\end{equation}

We use GPT-2 (small) \cite{radford2019language} as $\phi$. We fine-tune it (rather than training from scratch) on plausible human-annotated explanations from e-SNLI-VE, Hateful Memes, and VQA-X datasets, as oracle explanations for the classifier's decisions are unavailable. Directly conditioning on raw input $z$ leads to unfaithful outputs, as pretrained LLMs cannot inherently incorporate internal model states. As shown in Table~\ref{tab:baselines}, naïve fine-tuning on $(z, \text{explanation})$ pairs results in poor faithfulness. Our approach overcomes this by explicitly grounding generation in $c = E(z)$.

We denote the full explanation generation pipeline ($F = \mathcal{A} \circ \phi \circ \psi$), as the \textit{LLM machinery}. Thus, \emph{Explainer} is composed of both \emph{LLM Machinery} and $\mathcal{C}_2$ ($\mathcal{C}_2 \circ F$).

\paragraph{Aggregator $\mathcal{A}$}
The logit output of $\phi$ is a tensor of shape $(1, T, V)$, where $V$ is the vocabulary size and $T$ is the sequence length. We first sum over the time axis to obtain a vector $x \in \mathbb{R}^{V}$. This is then transformed via a feed-forward network into a vector of size $\mathbb{R}^{m\times1}$ to match the input dimension expected by $\mathcal{C}_2$.

\paragraph{Classifier $\mathcal{C}_2$}
$\mathcal{C}_2$ is a replica of $\mathcal{C}_1$, receiving the previously aggregated explanation  representation from $\mathcal{A}$ as input. It is trained to (i) match $\mathcal{C}_1$'s prediction $y_1$ via a teacher-student loss, and (ii) ensure causal alignment through the IIT loss (see Section~\ref{cabs}).

During training, the entire Explainer is optimized, with each component serving a distinct purpose (detailed in the next subsection). At inference, only the fine-tuned $\phi$ is used.

\begin{figure*}[h]
    \centering
    \includegraphics[width=\linewidth]{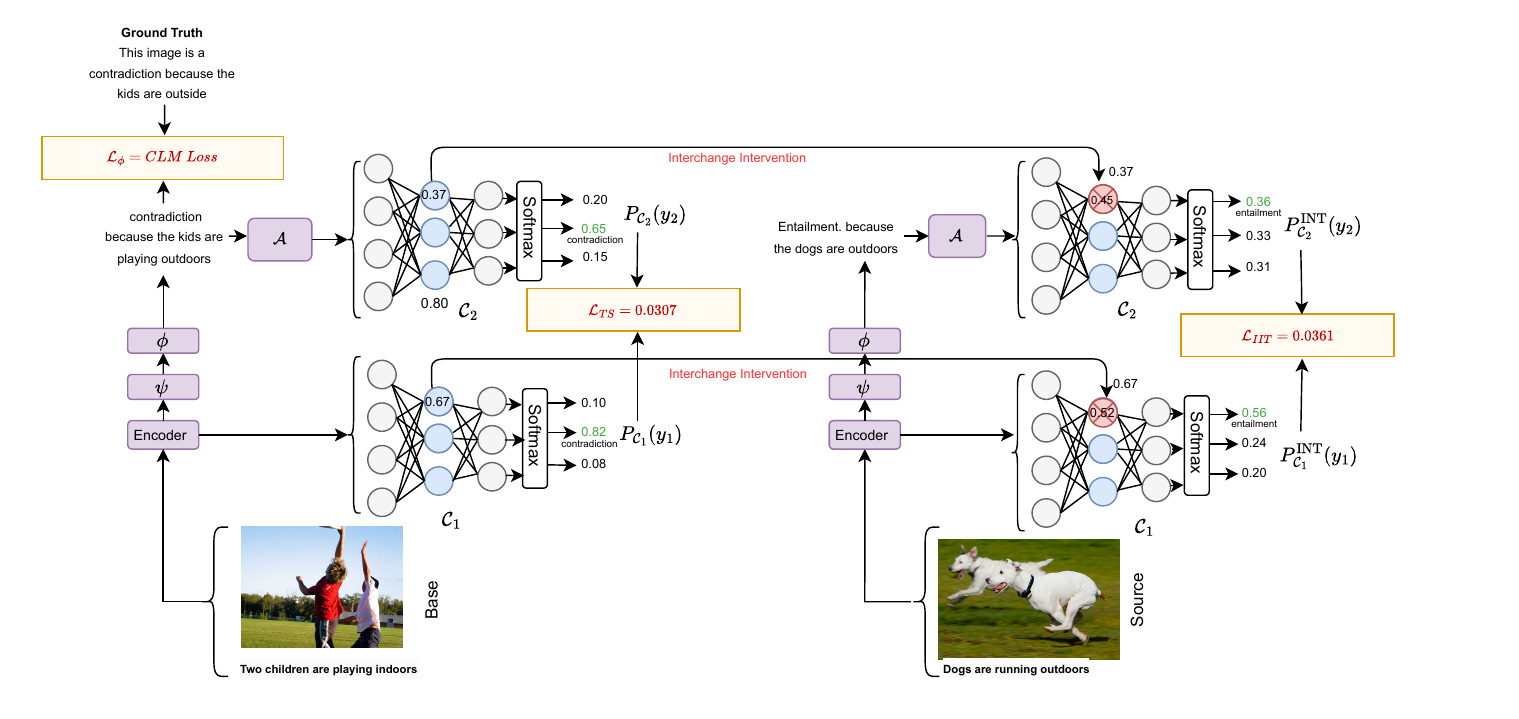}
    \caption{This \emph{toy} diagram shows the training process of the CAuSE framework, specifically the calculations of $\mathcal{L}_{TS}$ and $\mathcal{L}_{IIT}$ losses on a sample from e-SNLI-VE dataset. The input vectors to $\mathcal{C}_1$ and $\mathcal{C}_2$ are arbitrary but depicts a real scenario.}
    \label{fig:iit}
\end{figure*}
\subsection{Causal Alignment and Behavioral Simulation}\label{cabs}

The central goal of CAuSE is to ensure that the explanations produced by the Explainer are faithful to the behavior of the pretrained multimodal classifier $M = \mathcal{C}_1 \circ E$. This is accomplished in two steps: first, we make the Explainer simulate the decision of $M$; second, we ensure that the Explainer serves as a causal abstraction of $M$. Below, we define the core concepts and describe how CAuSE achieves these goals.

\paragraph{Simulation of Classifier Behavior}

We first formalize the notion of simulation.

\textbf{Definition (Simulation).} Two neural networks $A$ and $B$ are said to simulate each other if, when given the same input $z$, they produce the same output: $A(z) \approx B(z)$.

In our case, the encoder $E$ produces a multimodal representation $c = E(z)$ from an input $z = (t, v)$. This representation is passed directly to $\mathcal{C}_1$ to obtain a prediction $y_1 = \mathcal{C}_1(c)$. The same representation $c$ is also passed to the Explainer module.

To ensure that the Explainer mimics $\mathcal{C}_1$ on the same input, we train $\mathcal{C}_2$ using a teacher-student loss:
\begin{equation}
\mathcal{L}_{TS} = D_{KL}(P_{\mathcal{C}_1}(y_1) \| P_{\mathcal{C}_2}(y_2))
\end{equation}
By minimizing $\mathcal{L}_{TS}$, we ensure that $y_2=\mathcal{C}_2(F(c)) \approx \mathcal{C}_1(c)=y_1$. Since both systems receive the same input $c = E(z)$, and produce the same output, the Explainer (i.e., $\mathcal{C}_2 \circ F$) simulates the classifier $\mathcal{C}_1$ under the definition above.

\paragraph{Causal Abstraction via Identical Networks}

While simulation ensures output-level agreement, it does not guarantee that the Explainer behaves like $M$ under counterfactual interventions. For this, we seek a stronger condition: that the Explainer forms a \textit{causal abstraction} of $M$. To establish this, we first introduce the notion of causally identical neural networks.

\textbf{Definition (Causally Identical Networks).} Given two neural networks $A$ and $B$ of the same architecture, we say they are identical if: i) they share the same weights: $\mathcal{W}_A = \mathcal{W}_B$, and ii) they are trained under the IIT objective that enforces alignment of their outputs under counterfactual substitutions (Equation \ref{iitloss}). Under the definition posed in Section \ref{ii}, they become causal abstraction of each other.

To make $\mathcal{C}_2$ identical to $\mathcal{C}_1$ using the above definition, we use two objectives:
i) a weight-regularization loss to encourage weight similarity (using their frobenius norm):
    \begin{equation}
        \mathcal{R}_{\text{match}} = \|\mathcal{W}_{\mathcal{C}_1} - \mathcal{W}_{\mathcal{C}_2}\|_F
    \end{equation}
and ii) An IIT loss that encourages counterfactual agreement:
    \begin{equation}
        \mathcal{L}_{IIT} = D_{KL}(P^{\text{INT}}_{\mathcal{C}_1}(y_1) \| P^{\text{INT}}_{\mathcal{C}_2}(y_2))
    \end{equation}

When both $\mathcal{R}_{\text{match}}$ and $\mathcal{L}_{IIT}$ are minimized, $\mathcal{C}_2$ becomes \emph{identical} to $\mathcal{C}_1$, by virtue of the above definition.

\paragraph{From Classifier Alignment to Explainer Abstraction.}

We now connect the alignment between $\mathcal{C}_1$ and $\mathcal{C}_2$ to the abstraction property of the Explainer.

\textbf{Theorem.} If $\mathcal{C}_2$ is identical to $\mathcal{C}_1$ (as per the definition above), then the Explainer, comprising the LLM machinery and $\mathcal{C}_2$, becomes a causal abstraction of the pre-trained classifier $M = \mathcal{C}_1 \circ E$.

\textit{Proof Sketch.}
In Lemma \ref{theo3}, we prove that under $\mathcal{L}_{IIT}$, $F(E(z)) = E(z)$.

We consider three types of causal interventions: (i) \textbf{Input-level interventions}: Between $E$ and $F$, we assume a mapping $\alpha(z) = E(z)$. Given that $\mathcal{C}_2(F(E(z))) = \mathcal{C}_1(E(z))$ (by Lemma \ref{theo3}), the outputs from the model $M$ and the explainer remain consistent under such interventions. (ii) \textbf{Intermediate representation interventions}: Since $F$ acts as the identity function on $E(z)$ (by Lemma \ref{theo3}), interventions on $E(z)$ and $F(E(z))$ lead to equivalent outputs. (iii) \textbf{Neuron-level interventions}: When interventions are applied inside $\mathcal{C}_1$ and $\mathcal{C}_2$, the outputs stay aligned due to the shared architecture and training based on the IIT objective. Taken together, these cases demonstrate that the explainer and the model $M$ are causally equivalent under aligned interventions.

\begin{proof}
See Appendix \S\ref{proofs} for formal details.
\end{proof}

This establishes that not only does the Explainer match the classifier's outputs on observed examples (via simulation posed by $\mathcal{L}_{TS}$), but also that it faithfully reflects the causal structure of the $M$.

\paragraph{Summary.} CAuSE theoretically ensures both behavioral and causal alignment between the Explainer and the classifier: i) $\mathcal{L}_{TS}$ ensures the Explainer simulates $M$'s decision by matching predictions on the same encoder representation. ii) $\mathcal{L}_{IIT}$ and $\mathcal{R}_{\text{match}}$ make $\mathcal{C}_2$ causally identical to $\mathcal{C}_1$. This makes the overall Explainer a causal abstraction of the multimodal classifier $M = \mathcal{C}_1 \circ E$ (because of the result of the theorem posed above).

\paragraph{Overall Loss Function}
The final loss is:
\begin{align}\label{final-loss}
\mathcal{L}_{\text{CAuSE}} = \mathcal{L}_{\phi} +  \mathcal{L}_{TS} +  \mathcal{L}_{IIT} +  \mathcal{R}_{\text{match}}
\end{align}
Through a hypothetical example, the calculation of several components of the loss function is depicted in Figure \ref{fig:iit}. A walkthrough of the Figure is described in detail in Section \ref{walk_through_eg1} for reference.

\paragraph{Overall Training and Inference Protocol.}
i) We assume the base classifier $M = (\mathcal{C}_1 \circ E)$ is already trained until convergence. Then we freeze its weights. \emph{Training the classifier $M$ is not a part of CAuSE training framework.}
ii) We separately train the Explainer $(\psi, \phi, \mathcal{A}, \mathcal{C}_2)$ using $\mathcal{L}_{\text{CAuSE}}$, keeping $M$ frozen.
iii) At test time, we generate explanation $x$ from $\phi$ for any input $(t, v)$.

\paragraph{Clarification on Explainer Role.}
It is important to clarify that while the explainer is trained to simulate the predictions of $\mathcal{C}_1$, it is designed as a \emph{causal abstraction} of the entire classifier $M$. In our framework, we do not assume any inherent dependence between \emph{A being a causal abstraction of B} and \emph{A simulating B}. That is, causal abstraction and simulation are treated as independent properties. 

Nonetheless, we observe that combining simulation with causal abstraction yields better empirical results than simulation alone. Specifically, the mimicking performance of the explainer, quantified using both F1 and CCMR, improves in majority of the cases when the causal abstraction losses $\mathcal{L}_{\text{IIT}}$ and $\mathcal{R}_{\text{match}}$ are added to the simulation loss $\mathcal{L}_{\text{TS}}$ (also referred to as $\mathcal{L}_{\text{CAuSE}}$), as opposed to using $\mathcal{L}_{\text{TS}}$ alone (\S Table \ref{tab:cfs}).


\paragraph{Broad Neural Coverage under $\mathcal{L}_{IIT}$.}

At each step of minimizing $\mathcal{L}_{IIT}$, we sample $20\%$ of neurons uniformly randomly from $\mathcal{C}_1$ and $\mathcal{C}_2$ simultaneously and perform interchange intervention on them. Sampling a subset of neurons rather than performing interchange intervention on all of them is the norm in literature \cite{DBLP:journals/corr/abs-2106-02997} as that would prevent the network from being degenerate-where the network always outputs values for base input. Sampling a fraction of neurons at each backpropagation step probabilistically ensures broad coverage of network’s internal representation. With a very high probability, after a finite number of backpropagation steps, all the neurons in  $\mathcal{C}_1$ and $\mathcal{C}_2$ would be sampled at least once and interchange intervention would be done on all of them. Details of such a bound can be found in the Appendix Section \ref{bound}.

\subsection{Walk-through Example} \label{walk_through_eg1}

Figure~\ref{fig:iit} illustrates the training process of CAuSE on an e-SNLI-VE sample, involving three loss components: $\mathcal{L}_\phi$, $\mathcal{L}_{TS}$, and $\mathcal{L}_{IIT}$. Note that the following values are illustrative and derived from a simplified toy version of the CAuSE components, but the procedure remains identical in the full framework.

\paragraph{Step 1: Causal Language Modeling Loss ($\mathcal{L}_\phi$)}  
The base input is passed through the frozen encoder of $M$ to obtain the multimodal representation $E(z)$. This is then projected via $\psi$ and provided to the language model $\phi$, which generates a natural language explanation. The explanation generation step is supervised using a standard causal language modeling loss as shown in Equation \ref{clm}.

\paragraph{Step 2: Teacher–Student Loss ($\mathcal{L}_{TS}$)}  
The explanation from $\phi$ is passed through the aggregator $\mathcal{A}$ and input to classifier $\mathcal{C}_2$. Let the output distributions from $\mathcal{C}_1$ and $\mathcal{C}_2$ be:
\[
P_{\mathcal{C}_1}(y_1) = [0.10,\ 0.82,\ 0.08]
\]
\[
P_{\mathcal{C}_2}(y_2) = [0.20,\ 0.65,\ 0.15]
\]
Using KL divergence, the teacher–student loss is:
\[
\mathcal{L}_{TS} = D_{\mathrm{KL}}(P_{\mathcal{C}_1}(y_1)\,\|\,P_{\mathcal{C}_2}(y_2)) = 0.0307
\]

\paragraph{Step 3: Interchange Intervention Loss ($\mathcal{L}_{IIT}$)}  
The source input is encoded similarly. During this pass, neuron activations from the base pass (e.g., 0.67 from $\mathcal{C}_1$ and 0.37 from $\mathcal{C}_2$) are used to interchange (overwrite) corresponding activations for the source input. This is called interchange intervention. The resulting predictions after intervention are:
\[
P^{\text{INT}}_{\mathcal{C}_1}(y_1) = [0.56,\ 0.24,\ 0.20]
\]
\[
P^{\text{INT}}_{\mathcal{C}_2}(y_2) = [0.36,\ 0.33,\ 0.31]
\]
The IIT loss is:
\[
\mathcal{L}_{IIT} = D_{\mathrm{KL}}(P^{\text{INT}}_{\mathcal{C}_1}(y_1)\,\|\,P^{\text{INT}}_{\mathcal{C}_2}(y_2)) = 0.0361
\]

\paragraph{Training and Inference}  
Only the explainer modules ($\psi$, $\phi$, $A$, $\mathcal{C}_2$) are updated. The base classifier $M = \text{Encoder} + \mathcal{C}_1$ remains frozen. During inference, we use only $\psi$ and $\phi$ to generate faithful explanations conditioned on $E(z)$.

\subsection{Generalizability and Retraining in CAuSE} \label{retraining}

CAuSE is dataset-agnostic and requires no significant architectural changes across tasks. Its core components, $\phi$, Aggregator $\mathcal{A}$, and $\mathcal{C}_2$; are both \textit{necessary} and \textit{sufficient}: necessary, because each component is essential to ensure causal abstraction between the classifier $M$ and the explainer; and sufficient, because no structural modifications are needed when switching to a new dataset or classifier.

Minor adjustments include: (i) setting the output dimension of $\mathcal{C}_2$ to match the number of labels (e.g. 2 for Hateful Memes; 3 for e-SNLI-VE; 3,129 for VQA-X), and (ii) adapting $\psi$ to the hidden size of $c = E(x)$. Retraining is required for each (dataset $D$, classifier $M$) pair, as CAuSE is designed to explain a specific frozen model $M$ trained on $D$.

This retraining is lightweight, only $\sim$450M parameters trained and inference uses only $\sim$270M parameters, much smaller than large VLMs like PaLiGemma~\cite{beyer2024paligemmaversatile3bvlm} or LLaVA~\cite{liu2023improved}, which still underperform CAuSE (Table~\ref{tab:baselines}).

Overall, CAuSE is broadly applicable as long as hidden states $E(z)$ from $M$ and ground-truth explanations for $D$ are available. Adaptation to different tasks requires minimal effort with little architectural changes.

\section{Counterfactual Consistency via Multimodal Representation (CCMR)}\label{ccmr-score}

Existing faithfulness metrics such as CCT~\cite{siegel2024probabilitiesmatterfaithfulmetric} and Explanation Mention~\cite{atanasova-etal-2023-faithfulness} are designed for unimodal text models and rely on discrete text perturbations to test whether changes in input lead to corresponding changes in explanations. However, extending this to multimodal settings is nontrivial: there is no clear notion of discrete perturbation for image inputs, and fused text-image representations make it hard to isolate the effect of individual modalities. A text change may not be reflected in the explanation even if the model’s decision process is faithful, due to how multimodal fusion influences model behaviour. To address this, we propose CCMR, a metric based on continuous perturbations in multimodal representation space, enabling more reliable faithfulness evaluation.

\subsection{CCMR Score Calculation}\label{ccmt-score}

\textit{Counterfactual Faithfulness~\cite{atanasova-etal-2023-faithfulness}} : 
Consider an input text \( Z = \{w_0, w_1, \dots, w_i, \dots, w_{n-1}, w_n\} \), where \( w_i \) is the \( i \)-th word or token. A counterfactual input \( Z^C \) is created by replacing \( w_i \) with \( w^c_i \):  
\[
Z^C = \{w_0, w_1, \dots, w^c_i, \dots, w_{n-1}, w_n\}.
\]  
For a model, which generates natural language explanations (NLEs), let \( X \) and \( X^C \) be the NLEs for \( Z \) and \( Z^C \), respectively. The NLE \( X \) is considered faithful if \( w^c_i \in X^C \), meaning the explanation reflects the discrete input change.  

This test has two limitations: i) it does not support continuous input changes, and ii) \( X^C \) may not be a proper counterfactual, as it might not change the model's prediction class. We address these issues by constructing counterfactual representations as follows:  

Let \( z \in \mathcal{T} \) be a vector representation of a data point from the test set. The corresponding counterfactual input \( z' \) for $M$ satisfies:  
\[
z' = \arg\min_{z' \in \mathcal{T}} d(z, z') \text{ s.t. } \mathcal{C}_1(E(z)) \neq \mathcal{C}_1(E(z'))
\]

where \( d \) is Euclidean distance metric, and \( \mathcal{C}_1(z) \) denotes the output class of $M$ given input $z$. The counterfactual \( z' \) can be expressed as:  $z' = z + \mu$, where $\mu = z' - z$ is the perturbation. Given $z' =z +\mu$ as input to $M$, let us assume its output class is $y_1^{CF}$ .

Inspired by \citet{atanasova-etal-2023-faithfulness}, we define an explanation from $\phi$ as faithful if it reflects the change in $M$’s output under a counterfactual input. Let $E(z)$ be the input to the explainer $F$ when $M$ receives input $z$. We construct a potential counterfactual for the explainer as $E(z) + \nu$, and ensure it is valid for $M$ by solving a constrained optimization on $\nu$ such that $\mathcal{C}_1(E(z) + \nu) = \mathcal{C}_1(z + \mu)$.

Since $\mathcal{C}_2 \circ F$ is a causal abstraction of $M$, a counterfactual for $M$ should induce a corresponding counterfactual in the explainer. We input $E(z) + \nu$ to $F$ and obtain the predicted class from $\phi$ as $y_2^{CF}$. The Counterfactual Consistency via Multimodal Representation (CCMR) score is computed as the F1 score between $y_2^{CF}$ and $y_1^{CF}$, across test examples.

Importantly, we do not feed $z'$, a known counterfactual for $M$ from the test set, to the explainer. Doing so would let even a naive mimic of $M$ achieve high CCMR. Instead, we optimize for $\nu$ to construct a non-test representation of a counterfactual, ensuring the CCMR score reflects true causal consistency rather than mere imitation.

Algorithm \ref{alg:counterfactual} shows the calculation of the CCMR score and Table \ref{tab:cfs} shows the "faithfulness" performance of components of the the proposed CAuSE framework. 

\begin{algorithm}[!h]
\scriptsize
\caption{CCMR Score for the pre-trained classifier $M$ and the explainer}
\label{alg:counterfactual}
\KwIn{Data-point \( \mathbf{z} \in \mathcal{T} \)}

\SetKwFunction{FMain}{CounterFactual}
\SetKwProg{Fn}{Function}{:}{end}
\Fn{\FMain{$\mathbf{z}$}}{
    \( \mathbf{z'} \gets \arg\min_{\mathbf{z'} \in \mathcal{T}} d(\mathbf{z}, \mathbf{z'}) \) \text{ s.t. } \( \mathcal{C}_1(E(\mathbf{z})) \neq \mathcal{C}_1(E(\mathbf{z'})) \) \;
    \( \mu \gets \mathbf{z'} - \mathbf{z} \) \tcp*[r]{Compute the perturbation}

   \textbf{Optimize $\nu$ such that} \( \mathcal{C}_1(E(\mathbf{z}) + \mathbf{\nu}) = \mathcal{C}_1(E(\mathbf{z + \mathbf{\mu}})) \) 

    \KwRet{$E(\mathbf{z}) + \mathbf{\nu}$, $\mathbf{z'}$}
}

\SetKwProg{P}{Procedure}{:}{end}
\P{Calculate CCMR Score}{
    \SetKwData{ZList}{ZList}
    \SetKwData{XList}{XList}

    \ZList \(\gets \emptyset\) \;
    \XList \(\gets \emptyset\) \;

    \While{\( \mathcal{T} \neq \emptyset \)}{
        Sample \( \mathbf{p} \in \mathcal{T} \) \tcp*[r]{Draw a new data point}
        \( \mathbf{q'}, \mathbf{p'} \gets \FMain(\mathbf{p}) \) \;

        \ZList \(\gets \ZList \cup \{\mathcal{C}_2(F(\mathbf{q'}))\}\) \tcp*[r]{Append \( \mathcal{C}_2(F(\mathbf{q'})) \) to the list}
        \XList \(\gets \XList \cup \{\mathcal{C}_1(E(\mathbf{p'}))\}\) \tcp*[r]{Append \( \mathcal{C}_1(E(\mathbf{p'})) \) to the list}

        \( \mathcal{T} \gets \mathcal{T} - \{\mathbf{p}\} \) \;
    }

    \KwRet{$F_1(\XList, \ZList)$} \tcp*[r]{Return F1 score between \( \XList \) and \( \ZList \)}
}
\end{algorithm}

\textbf{Composite CCMR.} Varying $\nu$ may push $E(z) + \nu$ to become an out-of-distribution (OOD) sample for the explainer. This is especially true if the explainer is not a causal abstraction of $M$. In such cases, $F$ may fail to generate coherent predictions containing any of the output classes. Table \ref{tab:cfs} reports the percentage of feasible generations and their corresponding CCMR (across faithful generations). The composite CCMR is defined as the harmonic mean of these two and reflects the true causal faithfulness of CAuSE and its counterparts.

\section{Results and Analysis}\label{res-all}

\subsection{Dataset and Experimentation}\label{ds-exp}

\textbf{Datasets.} The e-SNLI-VE and VQA-X datasets provide human-annotated explanations for each image–text pair, which we use as ground-truth instances for training CAuSE. In contrast, the Facebook Hateful Meme (HM) dataset lacks such annotations. To address this, we generate synthetic explanations using the GPT-4o\footnote{\url{https://chatgpt.com/}} language model for the offensive class, followed by manual verification and post-processing. These curated explanations serve as ground-truth for training CAuSE. The details of this data creation process are provided in Appendix~\ref{dataprep}.

However, not all instances are used for training (generating explanations). The availability of training explanations is filtered based on the accuracy of the underlying multimodal classifier $M$. For example, if $M$ achieves 70\% accuracy on a batch of 1,000 samples, only the 700 correctly predicted examples, along with their corresponding explanations, are retained for training CAuSE on ground-truth prediction and explanations. For 300 misclassified examples, CAuSE is trained only to mimic $M$'s prediction.
This filtering ensures that CAuSE learns to explain $M$’s actual predictions, not the dataset labels. This filtering procedure is described in Algorithm~\ref{alg:cause_filtering}.

\begin{algorithm}[h]
\scriptsize
\caption{Filtering Dataset for CAuSE}
\label{alg:cause_filtering}
\KwIn{Test dataset $\mathcal{D} = \{(x_i, y_i, e_i)\}_{i=1}^N$, classifier $M$}
\SetKwData{Dprime}{D'}
\KwOut{Filtered dataset $\Dprime$}

\SetKwProg{P}{Procedure}{:}{end}
\P{FilterCorrectPredictions}{
    
    \Dprime \(\gets \emptyset\) \;

    \ForEach{$(x_i, y_i, e_i) \in \mathcal{D}$}{
        $\hat{y}_i \gets M(x_i)$ \tcp*[r]{Predict using classifier}
        \If{$\hat{y}_i = y_i$}{
            \Dprime \(\gets \Dprime \cup \{(x_i, \hat{y}_i, e_i)\}\) \tcp*[r]{Keep correctly predicted samples}
        }
        \Else {
        \Dprime \(\gets \Dprime \cup \{(x_i, \hat{y}_i)\}\) 
        }
    }

    \KwRet{$\Dprime$} \tcp*[r]{Return filtered dataset}
}
\end{algorithm}

\textbf{Experimentation.} The experiments were performed on a Kaggle kernel with PyTorch version 2.1.2 and a single P-100 GPU, with a random seed of $42$ maintained for all runs. Additionally,  visual language model (VLMs) baselines were implemented using PEFT\footnote{\url{https://github.com/huggingface/peft}} and LoRA \citep{hu2021loralowrankadaptationlarge}. 
\begin{table}[ht]
    \centering
    \scriptsize
    \adjustbox{width=0.75\columnwidth}{\begin{tabular}{|c|c|c|}
        \hline
        \textbf{Dataset} & \textbf{Train Split} & \textbf{Test Split} \\
        \hline
        e-SNLI-VE & 9000 & 1000 \\
        \hline
        Hateful Memes & 6997 & 1000 \\
        \hline
        VQA-X & 5997 & 960 \\ \hline
    \end{tabular}}
    
    \caption{Train-test splits for e-SNLI-VE and Hateful Memes, and VQA-X datasets.}\label{tab:dataset_splits}
\end{table}

\subsection{Automatic Evaluation} CAuSE is evaluated across two verticals: i) Faithfulness, and ii) Plausibility. 

Faithfulness measures the alignment between predictions of the explainer (e.g. CAuSE or $\phi$ or $\phi +TS$) \footnote{CAuSE, $\phi$, and $\phi$+TS refer to explainers trained with the loss functions $\mathcal{L}_{CAuSE}$, $\mathcal{L}_{\phi}$, and $\mathcal{L}_{\phi} + \mathcal{L}_{TS}$, respectively.} and $M$ for a specific input. It is quantified by F1 score between CAuSE predictions and predictions from $M$. Under the counterfactual regime, the normal F1 score is replaced by CCMR score. The process of calculating CCMR score is already illustrated in Section \ref{ccmt-score}.

Plausibility measures the semantic similarity of the generated NLEs from the explainer (e.g. CAuSE or $\phi$ or $\phi +TS$) to the ground-truth (GT) explanations. Note that we do not have access to \emph{oracle} explanations of the classifier decisions. Ground-truth explanations work as a proxy for unavailable \emph{oracle} explanations. BLEU~\cite{papineni-etal-2002-bleu} and BERTScore~\cite{zhang2020bertscoreevaluatingtextgeneration} are used to measure the plausibility.

\textbf{Faithfulness-plausibility tradeoff.} \label{f-p_tradeoff}
GT explanations may not reflect how the classifier $M$ actually arrives at its decisions (i.e. \emph{oracle} explanations).
As a result, an explainer trained to be a causal abstraction of the $M$ (via $\mathcal{L}_{CAuSE}$) can achieve high predictive alignment with $M$, as measured by metrics like F1 or CCMR. However, its generated explanations may score lower on BLEU or BERTScore when compared to the human-provided GT explanations, since the classifier's decision-making process might differ from human reasoning. 

This leads to a \emph{faithfulness–plausibility tradeoff}, high faithfulness (to the classifier) may come at the cost of plausibility and vice versa. 
Ideally, if we had access to the classifier’s actual reasoning process as an oracle, we could directly measure alignment. 
In its absence, F1/CCMR captures faithfulness, while BLEU/BERTScore reflects plausibility. This tradeoff is evident from Tables \ref{tab:cfs} and \ref{tab:abl}. $\phi$ achieves higher plausibility (cf. Table \ref{tab:abl}) at the cost of lower faithfulness (cf. Table \ref{tab:cfs}) compared to CAuSE, perfectly illustrating the tension between these two metrics.

\textbf{An Illustration of tradeoff.} We illustrate the faithfulness-plausibility tradeoff on three datasets through specific examples, taking  Qwen-VL \cite{bai2023qwenvlversatilevisionlanguagemodel} as $M$. In these examples, $\phi$ generated explanations are plausible and align with GT explanation, but for predictions, they are unfaithful, as they do not match the output of the underlying base model $M$. Conversely, for the same examples, CAuSE produces \textit{informative} explanations that are faithful to the base model’s predictions, yet less aligned with GT explanations (less plausible). These contrasting behaviours clearly demonstrate the inherent tradeoff between faithfulness and plausibility. Figure \ref{fig:pf-tradeoff} illustrates these.

\textit{VQA-X:} In a VQA-X example, the model ($M$) answers “yes” to the question “Is a storm going on?”. $\phi$’s explanation (“the answer is in the sky above the clouds”) is highly plausible, as it resembles the human-provided explanation, but unfaithful because it omits $M$’s prediction. Conversely, CAuSE’s explanation (“yes because it is black and white”) is faithful by retaining the model’s output but less plausible. Notably, CAuSE’s output is more diagnostic, as it reveals that $M$ may be relying on a superficial shortcut (the image being black and white) rather than semantic understanding (presence of black clouds).

\textit{e-SNLI-VE:} For the hypothesis “The men are fishermen”, which $M$ predicts as an entailment, $\phi$'s explanation suggests a neutral outcome and thus does not reflect the model's label, indicating low faithfulness. In contrast, CAuSE’s explanation perfectly mirrors the model's prediction. Although $\phi$'s explanation is unfaithful, it is more plausible than CAuSE’s because it has greater similarity to the human explanation, from which the CAuSE explanation is completely divergent. This clearly shows the faithfulness-plausibility tradeoff.

\textit{Hateful Meme:} The gold label is not offensive, while the classifier $M$ predicts offensive. This mismatch truncates the gold explanation to match the predicted label (refer to Section \ref{ds-exp}, Algorithm \ref{alg:cause_filtering}), inflating the plausibility of $\phi$’s explanation (“not offensive because not offensive”) that is more similar to the gold explanation than CAuSE’s explanation. CAuSE instead matches the classifier output (“offensive because it promotes harmful stereotypes”), thus faithful. Despite lower similarity to the truncated gold explanation, CAuSE is more informative as it faithfully reveals the purported reasoning behind the misclassification of $M$.


\begin{table*}[t]
\centering

\adjustbox{width=\textwidth}{
\begin{tabular}{l|l|lclc|lclc|lclc}
\toprule
\textbf{$M$}               & \textbf{Ablations}  &\multicolumn{4}{c|}{\textbf{Hateful Meme}}&\multicolumn{4}{c}{\textbf{e-SNLI-VE}} & \multicolumn{4}{c}{\textbf{VQA-X}}\\ \cmidrule(l){3-14} 
                                   &                &\textbf{F1} (\%)& \textbf{CCMR (\%)} &  \textbf{\% gen.}&\textbf{Composite CCMR}&\textbf{F1} (\%)& \textbf{CCMR (\%)} & \textbf{\% gen.} &\textbf{Composite CCMR}  & \textbf{F1} (\%)& \textbf{CCMR (\%)} & \textbf{\% gen.} &\textbf{Composite CCMR}  \\ \midrule
\multirow{3}{*}{CLIP+MFB}      & $ \phi$&99.00& \textbf{83.00}&  10.20&18.16&93.69& \textbf{79.89}& 17.72 &42.90  &\textbf{84.16}&44.79 &\textbf{86.97}&59.13\\
                                   & $ \phi + TS$&98.80& 67.86&  5.71&10.53&81.08& 60.19& 31.47 &41.33  &76.88 &44.92 &85.10 &58.80\\ \cdashline{2-14}
                                   & $CAuSE$ 
                                    &\textbf{99.30}& 55.40&  \textbf{50.10}&\textbf{52.65}&\textbf{95.30}& 67.36& \textbf{87.68}&\textbf{76.19}  &74.53 &\textbf{46.27}&86.67&\textbf{60.37}\\ \midrule
\multirow{3}{*}{VisualBERT}       & $\phi$&99.50& 38.15&  39.58&38.85&98.80& 69.20& 51.65 &59.15  &87.50 &6.17 &70.93 &11.35\\
                                   & $\phi + TS$&\textbf{100.00}& 40.05&  41.35&40.69&94.89& 64.08& \textbf{74.48}&\textbf{68.89}  &79.80 &5.27 &80.12 &9.89\\ \cdashline{2-14}
                                   & $CAuSE$ 
                                    &99.40& \textbf{47.94}&  \textbf{55.63}&\textbf{51.50}&\textbf{98.91}& \textbf{69.40}& 59.09 &63.83  &\textbf{87.83}&\textbf{26.84}&\textbf{80.19}&\textbf{40.22}\\ \bottomrule
 \multirow{3}{*}{FLAVA}& $ \phi$& \textbf{99.40}& 36.36& 2.37& 4.45& 97.40& 83.77& 50.71& 63.18& \textbf{97.40}& 33.33& 55.00& 41.51\\
 & $ \phi + TS$& 96.90& 34.97& \textbf{64.00}& \textbf{45.23}& 97.40& 82.92& \textbf{57.11}& 67.64& 83.02& 19.52& 56.56& 29.02\\
 \cdashline{2-14}& $CAuSE$ & 99.30& \textbf{49.68}& 16.76& 25.06& \textbf{97.70}& \textbf{87.20}& 56.40& \textbf{68.50}& 80.44& \textbf{35.22}& \textbf{62.08}& \textbf{44.94}\\ \midrule
 \multirow{3}{*}{Qwen-VL}&                                     $\phi$& 68.60& 38.49& 99.03& 55.43& 93.29& 65.10& 97.07& 77.93& 59.79& 10.44& \textbf{74.79}& 18.32\\
 & $ \phi + TS$& 61.30& 48.79& 99.87& 65.56& 85.59& 63.29& 77.33& 69.61& 26.23& 4.42& 70.84& 8.32\\
 \cdashline{2-14}& $CAuSE$ & \textbf{70.60}& \textbf{49.28}& \textbf{100.00}& \textbf{66.02}& \textbf{94.09}& \textbf{68.00}& \textbf{99.65}& \textbf{80.84}& \textbf{61.15}& \textbf{18.40}& 65.10& \textbf{28.69}\\ \midrule
\end{tabular}}
\caption{\emph{Faithfulness} Results for various classifiers ($M$), using CLIP+MFB, VisualBERT, FLAVA, and Qwen-VL. The table reports the effect of ablating individual loss components from Equation \ref{final-loss}.}

\label{tab:cfs}
\end{table*}

\begin{figure*}[t]
    \centering
    \includegraphics[height=7cm, keepaspectratio]{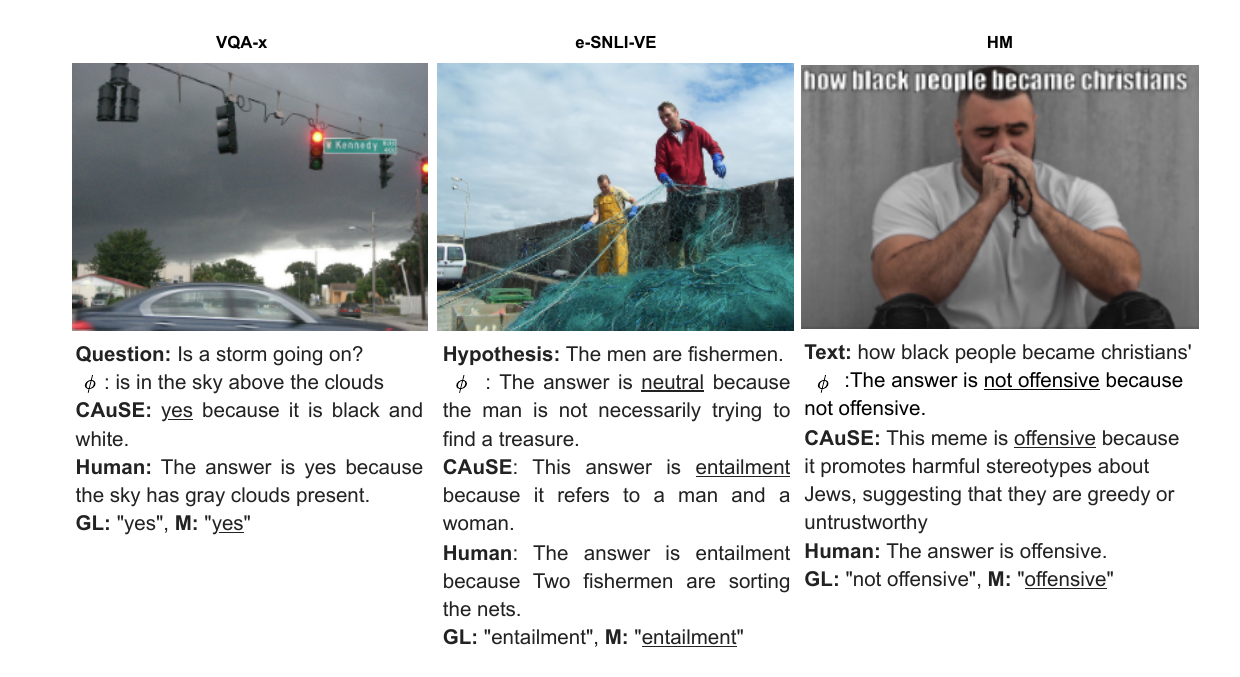}
    \caption{Examples illustrating the faithfulness–plausibility trade-off across the three datasets.
In each case, $\phi$ produces a more plausible explanation for an incorrect prediction (i.e. unfaithful), while CAuSE generates a less plausible but more faithful (reflecting a correct prediction) explanation. Both $\phi$ and CAuSE seek to emulate the classifier’s output and provide justifications consistent with the predicted answer. \textbf{GL:} Ground-truth label, \textbf{M:} Prediction from $M$.}
    \label{fig:pf-tradeoff}
\end{figure*}

\subsection{Baselines}\label{baselines}

\textbf{VLM baselines.}
Post-hoc interpretability techniques typically focus on identifying implicit (e.g., Integrated Gradients~\cite{sundararajan2017axiomaticattributiondeepnetworks}) or explicit concepts (e.g., Semantify~\cite{ijcai2024p684}) rather than generating coherent natural language explanations (NLEs). Visual language models (VLMs) can leverage zero-shot/few-shot prompting to generate plausible NLEs yet they are not guaranteed to be faithful~\cite{turpin2023language,madsen2024selfexplanations}.

CAuSE differs from prompting and fine-tuning VLMs by leveraging the hidden state of a multimodal encoder as the initial token. This ensures the generated NLEs to have been directly informed by the classifier’s internal decision process, making them more faithful and aligned with the classifier’s behavior, a capability absent in naive prompting or fine-tuning.

Nevertheless, we evaluate prompt based and fine-tuned VLMs as baselines for comparison with CAuSE. These models perform significantly worse across all metrics (see Table~\ref{tab:baselines}) \footnote{Prompts are defined in \url{https://github.com/newcodevelop/CAuSE/blob/main5/prompts.pdf}}. 

\textbf{Performance of Baselines.}

\textit{(a) Naive Prompting:} We evaluated zero-shot and few-shot ($k=0/2$) prompting using PaLiGemma and LLaVA, where VLMs were prompted with in-context examples showing classifier $M$'s decisions. However, the generated explanations showed poor faithfulness on Hateful Memes and e-SNLI-VE. Though the models showed improved faithfulness for VQA-X, the plausibility scores were very poor.

\textit{(b) Fine-Tuning:} Fine-tuning VLMs to imitate the classifier $M$ led to substantial improvements over prompting. These models are trained to reproduce $M$’s outputs given the same inputs, effectively learning to approximate its decision function.


\begin{table}[h]
\centering

\adjustbox{width=\columnwidth}{\begin{tabular}{@{}l|l|l|ccccc@{}}
\toprule
\textbf{Dataset} & \textbf{$M$} & \textbf{Ablations} & \textbf{BLEU-1} & \textbf{BLEU-2} & \textbf{BLEU-3} & \textbf{BLEU-4} & \textbf{BERTScore} \\ \midrule
\multirow{12}{*}{\textit{Hateful Meme}} 
 & \multirow{3}{*}{CLIP+MFB} & $\phi$ & 0.70& 0.65& 0.62& \textbf{0.59}& 0.98\\
 & & $\phi + TS$& 0.67& 0.61& 0.57& 0.53& 0.97\\ \cdashline{3-8}
 & & $CAuSE$ & \textbf{0.70}& \textbf{0.65}& \textbf{0.62}& 0.58& \textbf{0.98}\\ \cmidrule(lr){2-8}
 & \multirow{3}{*}{VisualBERT} & $\phi$ & \textbf{0.71}& \textbf{0.68}& \textbf{0.65}& \textbf{0.63}& \textbf{0.97}\\
 & & $\phi + TS$& 0.71& 0.68& 0.65& 0.63& 0.97\\ \cdashline{3-8}
 & & $CAuSE$ & 0.67& 0.62& 0.59& 0.56& 0.96\\
 \cmidrule(lr){2-8}& \multirow{3}{*}{ FLAVA}& $\phi$ &\textbf{0.62}&\textbf{0.54}&\textbf{0.46}&0.38&0.93\\
  & & $\phi + TS$&0.53&0.46&0.40&0.34&0.91\\
 \cdashline{3-8}& & $CAuSE$ &0.61&0.52&0.44&0.36&\textbf{0.93}\\
 \cmidrule(lr){2-8}& \multirow{3}{*} {Qwen-VL}& $\phi$ &0.58&0.52&0.48&0.45&0.95\\
  & & $\phi + TS$&\textbf{0.68}&\textbf{0.62}&\textbf{0.57}&\textbf{0.54}&0.95\\
 \cdashline{3-8}& & $CAuSE$ &0.55&0.49&0.44&0.40&\textbf{0.95}\\ \midrule
 \multirow{12}{*}{\textit{e-SNLI-VE}}& \multirow{3}{*}{CLIP+MFB} & $\phi$ & \textbf{0.39}& \textbf{0.33}& \textbf{0.28}& \textbf{0.24}& \textbf{0.91}\\
  & & $\phi + TS$& 0.34& 0.28& 0.24& 0.20& 0.90\\ \cdashline{3-8}
 & & $CAuSE$ & 0.38& 0.32& 0.27& 0.24& 0.90\\ \cmidrule(lr){2-8}
 & \multirow{3}{*}{VisualBERT} & $\phi$ & \textbf{0.43}& \textbf{0.36}& \textbf{0.31}& \textbf{0.27}& \textbf{0.91}\\
 & & $\phi + TS$& 0.38& 0.32& 0.27& 0.23& 0.90\\ \cdashline{3-8}
 & & $CAuSE$ & 0.39& 0.32& 0.27& 0.23& 0.90\\ 
 \cmidrule(lr){2-8}& \multirow{3}{*}{ FLAVA}& $\phi$ &\textbf{0.43}&\textbf{0.36}&\textbf{0.31}&\textbf{0.27}&\textbf{0.91}\\
 & & $\phi + TS$&0.42&0.35&0.30&0.26&0.91\\
 \cdashline{3-8}& & $CAuSE$ &0.40&0.32&0.26&0.22&0.90\\
 \cmidrule(lr){2-8}& \multirow{3}{*} {Qwen-VL}& $\phi$ &\textbf{0.42}&\textbf{0.34}&\textbf{0.27}&\textbf{0.22} &\textbf{0.91}\\
 & & $\phi +TS$&0.39 &0.31 &0.24 &0.20&0.91\\
 \cdashline{3-8}& & $CAuSE$ &0.38 &0.29 &0.23 &0.18 &0.91\\ \midrule
 \multirow{12}{*}{\textit{VQA-X}}& \multirow{3}{*}{CLIP+MFB} & $\phi$ & \textbf{0.43}& \textbf{0.36}& \textbf{0.31}& \textbf{0.25}& \textbf{0.92}\\
  & & $\phi+ TS$& 0.17& 0.14& 0.11& 0.09& 0.88\\ \cdashline{3-8}
 & & $CAuSE$ & 0.24& 0.19& 0.15& 0.12& 0.88\\ \cmidrule(lr){2-8}
 & \multirow{3}{*}{VisualBERT} & $\phi$ & \textbf{0.57}& \textbf{0.51}& \textbf{0.46}& \textbf{0.41}& \textbf{0.94}\\
 & & $\phi + TS$& 0.33& 0.27& 0.22& 0.18& 0.90\\ \cdashline{3-8}
 & & $CAuSE$ & 0.33& 0.27& 0.23& 0.19& 0.90\\ 
 \cmidrule(lr){2-8}& \multirow{3}{*}{ FLAVA}& $\phi$ & \textbf{0.55}& \textbf{0.52}&\textbf{ 0.49}& \textbf{0.45}&\textbf{0.95}\\
 & & $\phi + TS$& 0.48& 0.40& 0.35& 0.30&0.90\\
 \cdashline{3-8}& & $CAuSE$ & 0.13& 0.11& 0.09& 0.07&0.83\\
 \cmidrule(lr){2-8}& \multirow{3}{*} {Qwen-VL}& $\phi$ & 0.31& \textbf{0.24}& \textbf{0.17}& \textbf{0.10}&\textbf{0.89}\\
 & & $\phi + TS$& 0.20& 0.12& 0.08& 0.05&0.85\\
 \cdashline{3-8}& & $CAuSE$ & \textbf{0.23}& 0.17& 0.11& 0.07&0.88\\ \midrule
\end{tabular}}

\caption{\emph{Plausibility} results for various classifiers ($M$) showing the role of ablating various loss components in Equation \ref{final-loss}.}
\label{tab:abl}
\end{table}


\subsection{Ablation studies}

\textbf{What is the role of loss functions other than $\mathcal{L}_{\phi}$?} 

In Table \ref{tab:cfs}, CAuSE achieves the highest F1 score across almost all experiments, except for a few exceptions (four out of twelve cases in total).

In terms of causal abstraction, as measured by Composite CCMR, CAuSE also obtains the highest scores in most cases, with a few exceptions of VQA-X and Hateful meme datasets (two out of twelve cases in total). However, for plausibility (cf. Table \ref{tab:abl}), CAuSE ranks second in most combinations highlighting the faithfulness-plausibility trade-off we have mentioned in section \ref{f-p_tradeoff}.

IIT ensures causal abstraction between $M$ and the explainer. Consequently, we observe CAuSE obtains a slightly higher composite CCMR score and F1 score compared to its counterparts ($\phi$, and $\phi+TS$) in most cases (cf. Table \ref{tab:cfs}). This result provides empirical evidence that our proposed loss, $\mathcal{L}_{IIT}$, improves causal faithfulness.

\textbf{Is the faithfulness of CAuSE classifier agnostic?}
Yes it is \emph{largely classifier agnostic} since it maintains faithfulness for all the evaluated models in most settings. We have used two representative models from two fusion mechanisms widely used in literature, CLIP+MFB for late-fusion and VisualBERT for early-fusion. Also, two modern classifiers are used: FLAVA \cite{singh2022flava} and Qwen-VL \cite{bai2023qwenvlversatilevisionlanguagemodel}\footnote{\url{https://huggingface.co/Qwen/Qwen2-VL-2B}}. Note that, we use Qwen-VL by replacing its language modeling (LM) head layer with a multi-layered perceptron (MLP) for classification. For all of these models, CAuSE obtains a better CCMR score compared to the counterparts ($\phi$ or $\phi + TS$) in the majority of cases. Refer to Table \ref{tab:cfs} for reference.

\begin{table}[h]
\centering
\scriptsize
\begin{tabular}{@{}llcc@{}}
\toprule
\textbf{Dataset} & \textbf{Model} ($M = $Qwen-VL) & \textbf{Relatedness} & \textbf{Fluency} \\
\midrule
\multirow{2}{*}{HM} & $\phi$ & 2.10& \textbf{3.13}\\
 & CAuSE & \textbf{2.23}& 2.90\\
\midrule
\multirow{2}{*}{e-SNLI-VE} & $\phi$ & \textbf{2.87}& \textbf{3.09}\\
 & CAuSE & 2.70& 2.91\\
\midrule
\multirow{2}{*}{VQA-X} & $\phi$ & 1.90& \textbf{3.05}\\
 & CAuSE & \textbf{2.53}& 2.55\\
\bottomrule
\end{tabular}
\caption{Human evaluation scores for relatedness and fluency across datasets. The model M is Qwen-VL.}
\label{tab:human_evaluation}
\end{table}

\textbf{Are explanations from various models useful?}
A critical question is whether CAuSE's explanations remain useful despite lower plausibility. To investigate this, two annotators independently evaluated 150 randomly sampled explanations from CAuSE and $\phi$ across three datasets using a 5-point Likert scale ($1$=poor, $2$=fair, $3$=good, $4$=very good, $5$=excellent) on two dimensions:

i) \textit{Relatedness:} It assesses whether an explanation provides logically coherent justification for the model $M$'s prediction, independent of human-like reasoning (or GT explanation). For instance, CAuSE's VQA-X explanation "yes because it is black and white" earned high relatedness by meaningfully exposing the model's shortcut reasoning, despite diverging from human explanations referencing dark clouds (cf. left-most image of Figure \ref{fig:pf-tradeoff}).

ii) \textit{Fluency:} It measures grammatical quality of the generated explanation.

Annotators accessed inputs and predictions of $M$ but were explicitly instructed not to consult ground-truth explanations to avoid bias towards human-like reasoning. Inter-annotator agreement was substantial for both metrics: Cohen's $\kappa = 0.68$ for relatedness and $\kappa = 0.76$ for fluency.

Table~\ref{tab:human_evaluation} shows CAuSE maintains comparable utility to $\phi$: slightly higher relatedness ($2.49$ vs. $2.29$ for $\phi$) but slightly lower fluency ($2.79$ vs. $3.09$  for $\phi$). The minimal score difference across both metrics indicates that CAuSE's enhanced faithfulness (as measured through CCMR and F1) incurs a negligible cost in perceived usefulness, validating that CAuSE explanations, which purportedly expose the model's true reasoning, retain value for understanding discriminative classifier decisions.

\subsection{Qualitative studies}

\begin{table}[h]
\centering

\adjustbox{width=\columnwidth}{\begin{tabular}{c|l|llllll}
\hline
\multicolumn{1}{l|}{\textbf{Dataset}} & \textbf{Baselines }        & \textbf{F1}    & \textbf{BLEU-1}  & \textbf{BLEU-2}  & \multicolumn{1}{l}{\textbf{BLEU-3}} & \multicolumn{1}{l}{\textbf{BLEU-4}} & \multicolumn{1}{l}{\textbf{BERTScore}} \\ \hline
\multirow{6}{*}{\textit{Hateful Meme}} & LLaVA ($k=0$)      & 58.44 & 0.090 & 0.010 & 0.010                     & 0.010                     & 0.889                          \\
                              & LLaVA ($k=2$)     & 46.55 & 0.120 & 0.020 & 0.010                     & 0.010                     & 0.864                          \\

& PaLiGemma ($k=0$)     & 45.51 & 0.110 & 0.020 & 0.010                     & 0.010                     & 0.856                          \\

& PaLiGemma ($k=2$)     & 47.89 & 0.170 & 0.040 & 0.020                     & 0.020                     & 0.878                          \\

                              & PaLiGemma (FT)     & 72.33 & \textbf{0.410} & 0.270 & 0.150                     & 0.090                     & 0.891                          \\
                              & LLaVA (FT) &   \textbf{72.38}    &   0.400   &  \textbf{0.270}    &      \textbf{0.170}                    &            \textbf{0.130}              &     \textbf{0.894}                           \\
                               \cmidrule(l){1-8}
\multirow{6}{*}{\textit{e-SNLI-VE}}    & LLaVA ($k=0$)      & 33.12 & 0.220 & 0.070 & 0.030                     & 0.020                     & 0.876                          \\
                              & LLaVA ($k=2$)  &  35.77  &   0.220    &   0.070   &   0.030   &    0.010                      &   0.869                                                       \\

                               & PaLiGemma ($k=0$)  &  38.78  &   0.270    &   0.090   &   0.030   &    0.010                      &   0.851                                                       \\

                                & PaLiGemma ($k=2$)  &  31.21  &   0.180    &   0.050   &   0.020   &    0.010                      &   0.861                                                       \\
                                
                              & PaLiGemma (FT)     &  \textbf{64.90}     & 0.190     &   0.040   &   0.010                       &      0.010                    &      \textbf{0.866}                          \\
                              & LLaVA (FT) &  64.29     &  \textbf{0.220}    & \textbf{0.080 }    &   \textbf{0.030}                       &              \textbf{0.020 }           &    0.859                            \\
                                \cmidrule(l){1-8}

\multirow{6}{*}{\textit{VQA-X}} & LLaVA ($k=0$)      & 88.12 & 0.045 & 0.022 & 0.011                     & 0.002                     & 0.642                          \\
                              & LLaVA ($k=2$)     &91.11  & 0.033 & 0.014 & 0.007                     & 0.004                     & 0.653                          \\

& PaLiGemma ($k=0$)     &35.05  & 0.004 & 0.001 & 0.001                     & 0.000                     & 0.653                          \\

& PaLiGemma ($k=2$)     &35.74  & 0.004 & 0.001 & 0.001                     & 0.000                     & 0.651                          \\

                              & PaLiGemma (FT)     & 94.08 & \textbf{0.112} & \textbf{0.034} & \textbf{0.026}                     & \textbf{0.024}                     & \textbf{0.712}                          \\
                              & LLaVA (FT) &  \textbf{94.11}     &   0.070   &  0.022    &      0.011                    &            0.008              &     0.679                           \\
                              \bottomrule              
\end{tabular}}
\caption{VLM-based baselines. FT denotes finetuned model.  $k=0/2$ shows $0/2$ shots prompting results.}
\label{tab:baselines}
\end{table}

\begin{figure*}
    \centering
    \includegraphics[width=\textwidth]{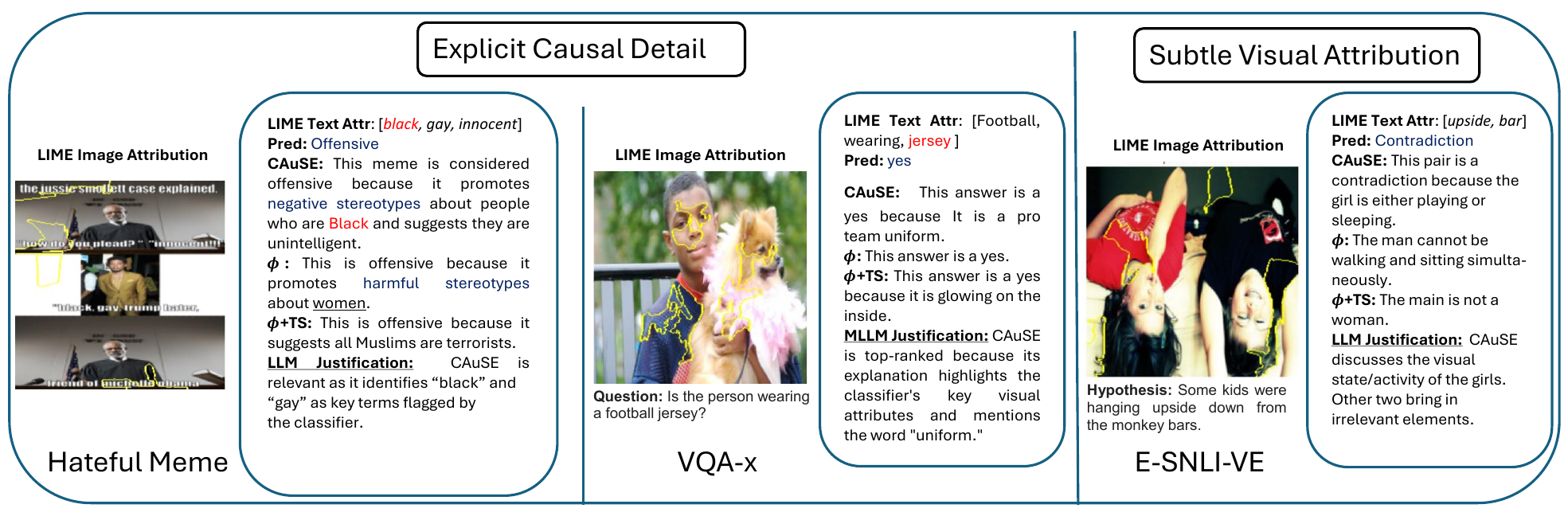}
    \caption{Representative dataset examples where CAuSE outperforms ablations in the LLM tournament.}
    \label{tab:llm_qualitative}
\end{figure*}

\begin{figure*}[h]
    \centering
    \small
    \includegraphics[width=\textwidth]{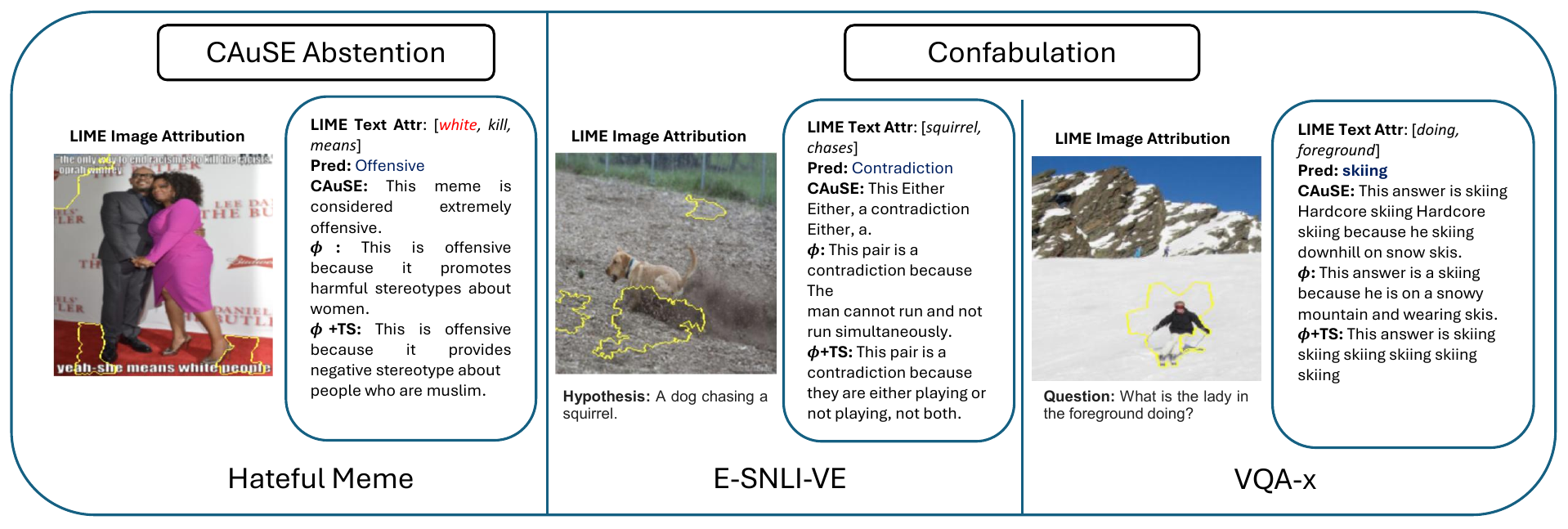}
    \caption{Representative dataset examples where CAuSE underperforms ablations in \textit{fluent}  and \textit{coherent} generation.}
    \label{tab:llm_qualitative2}
\end{figure*}

\begin{table*}[ht]
\centering
\renewcommand{\arraystretch}{1.7} 
\scriptsize

\resizebox{0.85\textwidth}{!}{\begin{tabular}{@{}p{0.15\textwidth} p{0.19\textwidth} p{0.17\textwidth} p{0.17\textwidth} p{0.19\textwidth} p{0.10\textwidth} p{0.10\textwidth}@{}}
\toprule
\textbf{Image} & \textbf{CAuSE} & $\phi$ only & $\phi$+TS & \textbf{Ground Truth (GT)} & \textbf{$M$} & \textbf{$\hat{y}$} \\
\midrule

\multirow{2}{=}{\centering\includegraphics[width=\linewidth]{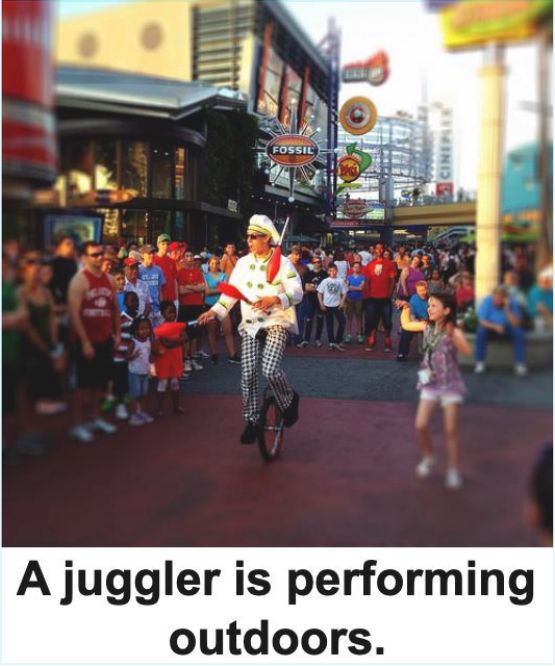}} 
& entailment because several people are standing in front of a building. 
& entailment because A man with a mustache is a type of man. 
& entailment because A man is a man. 
& A juggler is juggling clubs at an outdoor plaza. 
&  VisualBERT & Entailment \\ \cline{2-7}

& entailment because a man and running along a track is performing a stunt while jumping over a hurdle.
& entailment because A man is a man.
& entailment because The man is a man and the three people are the three people.
& A juggler is juggling clubs at an outdoor plaza.
& CLIP+MFB & Entailment \\
\midrule

\multirow{2}{=}{\centering\includegraphics[width=\linewidth]{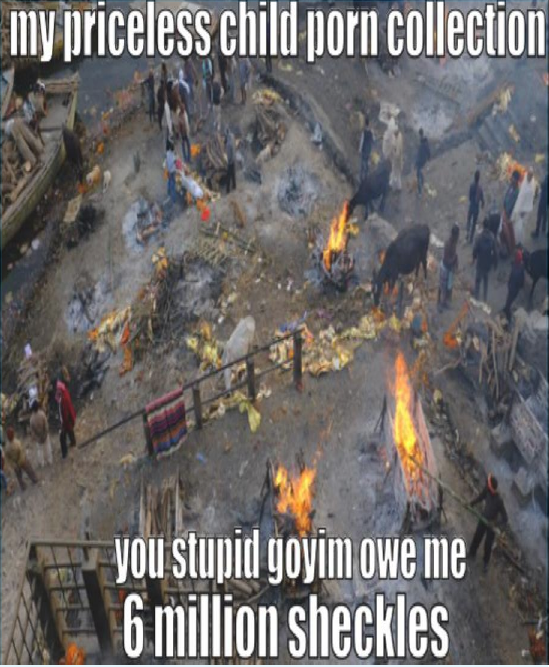}} 
& offensive because it promotes negative stereotypes about people who are Muslim.
& offensive because it promotes negative stereotypes about people who are Muslim.
& offensive because it promotes a harmful stereotype about Muslims.
& It promotes anti-Semitism and hatred towards Jewish people.
& VisualBERT & Offensive \\ \cline{2-7}

& This meme is offensive.
& This meme is offensive.
& This meme is offensive because it promotes a number of harmful stereotypes and prejudices.
& It promotes anti-Semitism and hatred towards Jewish people.
& CLIP+MFB & Offensive \\

\midrule

\multirow{2}{=}{\centering\includegraphics[width=\linewidth, height=2.9cm]{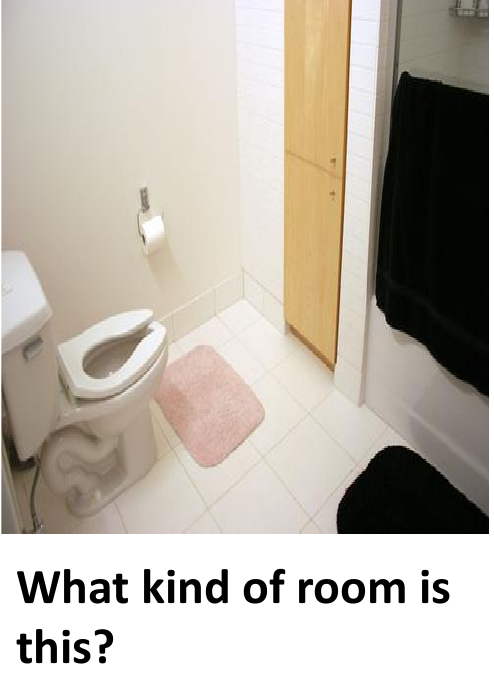}} 

& This answer bathroom is a bathroom because A toilet sits along the back wall. & This answer is a bathroom because There is a toilet and toilet paper. & This answer is a bathroom bath & This answer is a bathroom because I see a toilet and a tub. & VisualBERT & Bathroom

\\ \cline{2-7}

&The answer is bathroom because A toilet sits next to a sink. & The answer is bathroom because the toilet and sink are both in the room. & The answer is bathroom shower & This answer is a bathroom because I see a toilet and a tub. & CLIP+MFB & Bathroom \\
\bottomrule
\end{tabular}}
\caption{Error Analysis: These cases demonstrate errors that are prevalent in baselines as well as CAuSE. For all such cases, the classifier predicted class and the class predicted by the explainers are the same but the explanation does not \emph{completely} reflect the underlying scenario.}
\label{tab:error-analysis}
\end{table*}

\subsubsection{MLLM-as-a-Judge Evaluation}\label{llm-as-judge}

To better understand the relative strengths of CAuSE and its ablated variants, we adopt the \textit{MLLM-as-a-Judge} framework~\cite{chen2024mllmasajudgeassessingmultimodalllmasajudge}, which leverages a multimodal large language model (MLLM) to assess which of several candidate outputs best aligns with a reference model’s reasoning, which in this case is the frozen multimodal classifier $M$ (VisualBERT).

We conduct a tournament-style evaluation where the MLLM ranks natural language explanations generated by:
(i) \textit{CAuSE}: CAuSE framework with the full loss $\mathcal{L}_{CAuSE}$ ,
(ii) $\phi+TS$: CAuSE framework trained with $\mathcal{L}_\phi + \mathcal{L}_{TS}$ (i.e., without $\mathcal{L}_{IIT}$), and
(iii) $\phi$: CAuSE framework trained with $\mathcal{L}_\phi$ only.

For each test example, we first compute the classifier’s LIME~\cite{ribeiro2016whyitrustyou} attributions, highlighting the most influential text tokens and image regions used by $M$ to make its prediction. The MLLM is then given the input, the classifier’s predicted label, and the corresponding LIME attributions, along with the three candidate explanations. It is asked to rank the candidates based on how well they reflect the classifier’s prediction and saliency map.

Ranking-based scores are then assigned as follows:

(i) If the ranking is $CAuSE > \phi + TS > \phi$, the assigned scores are 2, 1, and 0, respectively. (ii) If $CAuSE = \phi + TS > \phi$, the scores are 1.5, 1.5, and 0. (iii) If $CAuSE > \phi + TS = \phi$, the scores are 2, 0.5, and 0.5, and (iv) If all explanations are judged equally good, each method receives a score of 1.

Final tournament scores are obtained by aggregating these per-sample scores across the entire test set. As shown in Table~\ref{win-rate}, CAuSE with the full loss consistently outperforms its ablations across all three datasets.

\begin{table}[h]
\centering
\scriptsize
\begin{tabular}{lccc}
\toprule
\textbf{Dataset} & \textbf{CAuSE} & \textbf{$\phi$ only} & \textbf{$\phi$+TS} \\
\midrule
HM (VisualBERT)    & \textbf{1.08} & 0.98  & 0.93  \\
eSNLI (VisualBERT) & \textbf{1.15} & 0.96 & 0.67 \\
VQA-X (VisualBERT) & \textbf{1.01} & 0.87 & 0.85 \\
\bottomrule
\end{tabular}
\caption{MLLM Tournament scores for the explainers on HM and eSNLI, and VQA-X datasets.}\label{win-rate}
\end{table}

\subsubsection{Qualitative Examples from the MLLM Tournament} \label{quals-disc}
To make it clear \emph{why} CAuSE wins, Figure~\ref{tab:llm_qualitative} shows three representative examples (\emph{with} $M$ as VisualBERT) where CAuSE was judged best by our blind MLLM critic. Each row gives: a LIME image heatmap attribution, the top important words from the text as per LIME, the classifier's prediction, three candidate explanations, and the MLLM's free-text justification. We can clearly observe that CAuSE can capture most of the nuances of the classifier attribution, denoted as \emph{Explicit causal detail}, and also point out the visual cues while generating the explanation, denoted as \emph{Subtle visual attribution}. For instance, CAuSE identifies that the girl is playing in the third image, which aligns with the LIME image attribution, highlighting the girl's hands in a playful mode.

\subsubsection{Failure Modes of CAuSE}\label{failure-mode}

In many scenarios, when subtle reasoning is needed which is not obvious from both textual and visual inputs, CAuSE fails to produce a coherent explanation. In Figure \ref{tab:llm_qualitative2}, we show three such cases. 

\paragraph{Illustrative Examples.}

In the first example, due to the implicit nature of offense (racism), CAuSE ablations fail to accurately reason why the meme is offensive (implicit racism), although correctly highlighting it was closely related to "stereotype". CAuSE does not explain anything at all except that the meme is offensive.

In the second example, the hypothesis states: "Dog chasing a squirrel." The image does not clearly indicate whether the dog is actually chasing a squirrel, making the ground-truth label inherently ambiguous. Although the classifier predicts contradiction, it does so with low confidence (probability of 0.52). As a causal abstraction of the classifier, CAuSE reflects this uncertainty in its explanation by using terms like "Either". However, it ultimately fails to generate a coherent explanation. In contrast, the other two variants, although mostly incorrect (especially $\phi$), manage to produce a more plausible explanation. The third example demonstrates CAuSE's poor fluency. Although its reasoning is correct and plausible, the explanation contains grammatical mistakes and is a clear case of confabulation. In contrast, $\phi$ generates more natural and fluent text.

\subsection{Generic Error Analysis}\label{error-analysis}

In Table \ref{tab:error-analysis}, we selected three examples, one from each dataset, to illustrate two broad categories of cases where both baselines and CAuSE fail in generating logically valid explanations.

\textbf{Lack of fine-grained representation.} 
In the first two examples from e-SNLI-VE dataset, the hypothesis states: “A juggler is performing outdoors,” and the premise is entailed, as supported by the ground-truth explanation: “A juggler is juggling clubs at an outdoor plaza.” While CAuSE and its counterparts correctly match the classifier’s prediction (both for VisualBERT and CLIP+MFB), the generated explanations fall short; CAuSE omits key details (e.g., the word “juggling”), and others produce incorrect outputs. These errors likely stem from limited object-level detail in the initial representation $c$ used by CAuSE. 
The last two examples from the VQA-X dataset exhibit a similar issue. Although CAuSE and comparable methods produce predictions consistent with the ground truth and generate largely accurate explanations, they fail to attribute the label to the presence of the bathtub, which is only faintly visible in the input image.

\textbf{Implicit semantic category.} In the third and fourth examples (from Hateful Meme dataset), although CAuSE correctly predicts the output class as offensive, it does so for the wrong reasons. This is true for other counterparts as well. Due to the explicit content in the meme, all the models, including CAuSE diagnose that the meme is offensive based on general stereotypes, but they could not point out "antisemitism", as neither the image nor the text explicitly conveys the historical context of the Holocaust, where six million Jews were killed. Without this implicit context, no model can pinpoint antisemitism present in this meme.

\section{Related Work}

\textbf{Interpretability}. Interpretability is crucial for building trust in AI systems within human society. Techniques like LIME, SHAP and RISE \citep{ribeiro2016whyitrustyou, lundberg2017unified, petsiuk2018riserandomizedinputsampling} explain classifier predictions by providing feature-level explanations for local interpretability. Although model-agnostic, these methods lack global interpretability, which is addressed by GALE \cite{vanderlinden2019globalaggregationslocalexplanations}, where local explanations are aggregated into a global model understanding. Approaches like SmoothGrad \cite{smilkov2017smoothgradremovingnoiseadding} and Integrated Gradients \citep{sundararajan2017axiomaticattributiondeepnetworks} utilize input gradients for model explanation, while CAM \cite{zhou2015learningdeepfeaturesdiscriminative} highlights critical pixels for decision making in visual classification. Counterfactual generations \citep{chang2019explainingimageclassifierscounterfactual, Mothilal_2020, goyal2019counterfactualvisualexplanations} also offer insights into the inner working of the model by revealing decision boundaries. However, most of these methods often overlook implicit features behind model decisions and lack natural language explanations. To address these limitations, we propose a novel framework for classifier explanation which generates both \textit{faithful} and \textit{plausible} (human-like) natural language outputs.

\textbf{Causal Interpretability}. Causal interpretability refers to the ability to explain a model's decisions by identifying the cause-effect relationships between input features and the model’s output. \citet{feder-etal-2022-causal} demonstrated how incorporating causal reasoning in NLP tasks can improve model predictions and enhance interpretability by going beyond simple correlations between input features and outputs. Further works by \cite{DBLP:journals/corr/abs-2106-02997,vig2020causal,meng2023locating} have focused on causal abstraction and causal mediation analysis, helping to create causally faithful models and identify both direct and indirect causal factors behind certain model behaviors. In addition to generating counterfactuals, testing models on counterfactual inputs is another critical aspect of understanding model behavior. Since creating exact counterfactuals is challenging, \cite{abraham2022cebab,calderon-etal-2022-docogen}, recent research has focused on approximations \cite{DBLP:journals/corr/abs-2106-02997} or counterfactual representations \cite{feder-etal-2021-causalm,elazar2021amnesic,ravfogel-etal-2021-counterfactual}. Our proposed counterfactual metric CCMR is inspired by these counterfactual representations. Moreover, most of the existing works focus on single modality \cite{feder-etal-2021-causalm,goyal2020explaining}. In contrast, the natural language explanation provided by our framework is model and task-agnostic, and capable of handling multimodal inputs.

\section{Conclusion and Future Work}

In this paper, we presented \textit{CAuSE} (Causal Abstraction under Simulated Explanation), a novel framework for generating causally faithful natural language explanations for discriminative multimodal classifiers. By integrating \textit{Interchange Intervention Training} (IIT) with a Language Model (LM) based module, CAuSE addresses the limitations of existing interpretability methods, ensuring explanations are directly tied to the classifier’s causal reasoning. Our proposed CCMR score highlights CAuSE's state-of-the-art performance on datasets like e-SNLI-VE, Hateful Memes, and VQA-X.

While CAuSE demonstrates robust task-agnostic performance, future work will focus on enhancing fine-grained object-level representations and extending the framework to temporal data, such as video and audio. Additionally, we aim to explore how \textit{self-supervised} learning and deeper integration of implicit cultural knowledge can further improve the framework’s scalability and contextual understanding in real-world applications. Currently, CAuSE is tailored for discriminative classifiers and adapting it for generative models is an area for future research.

\section*{Acknowledgements}
Dibyanayan Bandyopadhyay acknowledges support from the Prime Minister's Research Fellowship (PMRF). The authors also thank the anonymous reviewers and the Action Editors for their constructive and valuable feedback, which helped improve the quality of this work.

\bibliography{tacl2021}

@inproceedings{papineni-etal-2002-bleu,
    title = "{B}leu: a Method for Automatic Evaluation of Machine Translation",
    author = "Papineni, Kishore  and
      Roukos, Salim  and
      Ward, Todd  and
      Zhu, Wei-Jing",
    editor = "Isabelle, Pierre  and
      Charniak, Eugene  and
      Lin, Dekang",
    booktitle = "Proceedings of the 40th Annual Meeting of the Association for Computational Linguistics",
    month = jul,
    year = "2002",
    address = "Philadelphia, Pennsylvania, USA",
    publisher = "Association for Computational Linguistics",
    url = "https://aclanthology.org/P02-1040",
    doi = "10.3115/1073083.1073135",
    pages = "311--318",
}

@inproceedings{
zhang2020bertscoreevaluatingtextgeneration,
title={BERTScore: Evaluating Text Generation with BERT},
author={Tianyi Zhang* and Varsha Kishore* and Felix Wu* and Kilian Q. Weinberger and Yoav Artzi},
booktitle={International Conference on Learning Representations},
year={2020},
url={https://openreview.net/forum?id=SkeHuCVFDr}
}

@misc{beyer2024paligemmaversatile3bvlm,
      title={PaliGemma: A versatile 3B VLM for transfer}, 
      author={Lucas Beyer and Andreas Steiner and André Susano Pinto and Alexander Kolesnikov and Xiao Wang and Daniel Salz and Maxim Neumann and Ibrahim Alabdulmohsin and Michael Tschannen and Emanuele Bugliarello and Thomas Unterthiner and Daniel Keysers and Skanda Koppula and Fangyu Liu and Adam Grycner and Alexey Gritsenko and Neil Houlsby and Manoj Kumar and Keran Rong and Julian Eisenschlos and Rishabh Kabra and Matthias Bauer and Matko Bošnjak and Xi Chen and Matthias Minderer and Paul Voigtlaender and Ioana Bica and Ivana Balazevic and Joan Puigcerver and Pinelopi Papalampidi and Olivier Henaff and Xi Xiong and Radu Soricut and Jeremiah Harmsen and Xiaohua Zhai},
      year={2024},
      eprint={2407.07726},
      archivePrefix={arXiv},
      primaryClass={cs.CV},
      url={https://arxiv.org/abs/2407.07726}, 
}

@inproceedings{
liu2023visualinstructiontuning,
title={Visual Instruction Tuning},
author={Haotian Liu and Chunyuan Li and Qingyang Wu and Yong Jae Lee},
booktitle={Thirty-seventh Conference on Neural Information Processing Systems},
year={2023},
url={https://openreview.net/forum?id=w0H2xGHlkw}
}

@misc{do2021esnlivecorrectedvisualtextualentailment,
      title={e-SNLI-VE: Corrected Visual-Textual Entailment with Natural Language Explanations}, 
      author={Virginie Do and Oana-Maria Camburu and Zeynep Akata and Thomas Lukasiewicz},
      year={2021},
      eprint={2004.03744},
      archivePrefix={arXiv},
      primaryClass={cs.CL},
      url={https://arxiv.org/abs/2004.03744}, 
}

@inproceedings{
hu2021loralowrankadaptationlarge,
title={Lo{RA}: Low-Rank Adaptation of Large Language Models},
author={Edward J Hu and yelong shen and Phillip Wallis and Zeyuan Allen-Zhu and Yuanzhi Li and Shean Wang and Lu Wang and Weizhu Chen},
booktitle={International Conference on Learning Representations},
year={2022},
url={https://openreview.net/forum?id=nZeVKeeFYf9}
}

@inproceedings{kiela2021hateful,
author = {Kiela, Douwe and Firooz, Hamed and Mohan, Aravind and Goswami, Vedanuj and Singh, Amanpreet and Ringshia, Pratik and Testuggine, Davide},
title = {The hateful memes challenge: detecting hate speech in multimodal memes},
year = {2020},
isbn = {9781713829546},
publisher = {Curran Associates Inc.},
address = {Red Hook, NY, USA},
booktitle = {Proceedings of the 34th International Conference on Neural Information Processing Systems},
articleno = {220},
numpages = {14},
location = {Vancouver, BC, Canada},
series = {NIPS '20}
}

@inproceedings{madsen2024selfexplanations,
    title = "Are self-explanations from Large Language Models faithful?",
    author = "Madsen, Andreas  and
      Chandar, Sarath  and
      Reddy, Siva",
    editor = "Ku, Lun-Wei  and
      Martins, Andre  and
      Srikumar, Vivek",
    booktitle = "Findings of the Association for Computational Linguistics: ACL 2024",
    month = aug,
    year = "2024",
    address = "Bangkok, Thailand",
    publisher = "Association for Computational Linguistics",
    url = "https://aclanthology.org/2024.findings-acl.19/",
    doi = "10.18653/v1/2024.findings-acl.19",
    pages = "295--337"
}

@article{DBLP:journals/corr/abs-2106-02997,
  author       = {Atticus Geiger and
                  Hanson Lu and
                  Thomas Icard and
                  Christopher Potts},
  title        = {Causal Abstractions of Neural Networks},
  journal      = {CoRR},
  volume       = {abs/2106.02997},
  year         = {2021},
  url          = {https://arxiv.org/abs/2106.02997},
  eprinttype    = {arXiv},
  eprint       = {2106.02997},
  bibsource    = {dblp computer science bibliography, https://dblp.org}
}

@InProceedings{chattopadhyay2019neural,
  title = 	 {Neural Network Attributions: A Causal Perspective},
  author =       {Chattopadhyay, Aditya and Manupriya, Piyushi and Sarkar, Anirban and Balasubramanian, Vineeth N},
  booktitle = 	 {Proceedings of the 36th International Conference on Machine Learning},
  pages = 	 {981--990},
  year = 	 {2019},
  editor = 	 {Chaudhuri, Kamalika and Salakhutdinov, Ruslan},
  volume = 	 {97},
  series = 	 {Proceedings of Machine Learning Research},
  month = 	 {09--15 Jun},
  publisher =    {PMLR},
  url = 	 {https://proceedings.mlr.press/v97/chattopadhyay19a.html}
}

@inproceedings{
meng2023locating,
title={Locating and Editing Factual Associations in {GPT}},
author={Kevin Meng and David Bau and Alex J Andonian and Yonatan Belinkov},
booktitle={Advances in Neural Information Processing Systems},
editor={Alice H. Oh and Alekh Agarwal and Danielle Belgrave and Kyunghyun Cho},
year={2022},
url={https://openreview.net/forum?id=-h6WAS6eE4}
}

@inproceedings{vig2020causal,
 author = {Vig, Jesse and Gehrmann, Sebastian and Belinkov, Yonatan and Qian, Sharon and Nevo, Daniel and Singer, Yaron and Shieber, Stuart},
 booktitle = {Advances in Neural Information Processing Systems},
 editor = {H. Larochelle and M. Ranzato and R. Hadsell and M.F. Balcan and H. Lin},
 pages = {12388--12401},
 publisher = {Curran Associates, Inc.},
 title = {Investigating Gender Bias in Language Models Using Causal Mediation Analysis},
 url = {https://proceedings.neurips.cc/paper_files/paper/2020/file/92650b2e92217715fe312e6fa7b90d82-Paper.pdf},
 volume = {33},
 year = {2020}
}

@article{feder-etal-2022-causal,
    title = "Causal Inference in Natural Language Processing: Estimation, Prediction, Interpretation and Beyond",
    author = "Feder, Amir  and
      Keith, Katherine A.  and
      Manzoor, Emaad  and
      Pryzant, Reid  and
      Sridhar, Dhanya  and
      Wood-Doughty, Zach  and
      Eisenstein, Jacob  and
      Grimmer, Justin  and
      Reichart, Roi  and
      Roberts, Margaret E.  and
      Stewart, Brandon M.  and
      Veitch, Victor  and
      Yang, Diyi",
    editor = "Roark, Brian  and
      Nenkova, Ani",
    journal = "Transactions of the Association for Computational Linguistics",
    volume = "10",
    year = "2022",
    address = "Cambridge, MA",
    publisher = "MIT Press",
    url = "https://aclanthology.org/2022.tacl-1.66",
    doi = "10.1162/tacl_a_00511",
    pages = "1138--1158",
}

@inproceedings{abraham2022cebab,
author = {Abraham, Eldar David and D'Oosterlink, Karel and Feder, Amir and Gat, Yair and Geiger, Atticus and Potts, Christopher and Reichart, Roi and Wu, Zhengxuan},
title = {CEBaB: estimating the causal effects of real-world concepts on NLP model behavior},
year = {2022},
isbn = {9781713871088},
publisher = {Curran Associates Inc.},
address = {Red Hook, NY, USA},
booktitle = {Proceedings of the 36th International Conference on Neural Information Processing Systems},
articleno = {1278},
numpages = {15},
location = {New Orleans, LA, USA},
series = {NIPS '22}
}

@article{feder-etal-2021-causalm,
    title = "{C}ausa{LM}: Causal Model Explanation Through Counterfactual Language Models",
    author = "Feder, Amir  and
      Oved, Nadav  and
      Shalit, Uri  and
      Reichart, Roi",
    journal = "Computational Linguistics",
    volume = "47",
    number = "2",
    month = jun,
    year = "2021",
    address = "Cambridge, MA",
    publisher = "MIT Press",
    url = "https://aclanthology.org/2021.cl-2.13",
    doi = "10.1162/coli_a_00404",
    pages = "333--386",
}

@article{elazar2021amnesic,
    author = {Elazar, Yanai and Ravfogel, Shauli and Jacovi, Alon and Goldberg, Yoav},
    title = {Amnesic Probing: Behavioral Explanation with Amnesic Counterfactuals},
    journal = {Transactions of the Association for Computational Linguistics},
    volume = {9},
    pages = {160-175},
    year = {2021},
    month = {03},
    issn = {2307-387X},
    doi = {10.1162/tacl_a_00359},
    url = {https://doi.org/10.1162/tacl\_a\_00359},
    eprint = {https://direct.mit.edu/tacl/article-pdf/doi/10.1162/tacl\_a\_00359/1924189/tacl\_a\_00359.pdf},
}

@InProceedings{goyal2019counterfactualvisualexplanations,
  title = 	 {Counterfactual Visual Explanations},
  author =       {Goyal, Yash and Wu, Ziyan and Ernst, Jan and Batra, Dhruv and Parikh, Devi and Lee, Stefan},
  booktitle = 	 {Proceedings of the 36th International Conference on Machine Learning},
  pages = 	 {2376--2384},
  year = 	 {2019},
  editor = 	 {Chaudhuri, Kamalika and Salakhutdinov, Ruslan},
  volume = 	 {97},
  series = 	 {Proceedings of Machine Learning Research},
  month = 	 {09--15 Jun},
  publisher =    {PMLR},
  url = 	 {https://proceedings.mlr.press/v97/goyal19a.html}
}

@misc{goyal2020explaining,
      title={Explaining Classifiers with Causal Concept Effect (CaCE)}, 
      author={Yash Goyal and Amir Feder and Uri Shalit and Been Kim},
      year={2020},
      eprint={1907.07165},
      archivePrefix={arXiv},
      primaryClass={cs.LG}
}

@inproceedings{ravfogel-etal-2021-counterfactual,
    title = "Counterfactual Interventions Reveal the Causal Effect of Relative Clause Representations on Agreement Prediction",
    author = "Ravfogel, Shauli  and
      Prasad, Grusha  and
      Linzen, Tal  and
      Goldberg, Yoav",
    editor = "Bisazza, Arianna  and
      Abend, Omri",
    booktitle = "Proceedings of the 25th Conference on Computational Natural Language Learning",
    month = nov,
    year = "2021",
    address = "Online",
    publisher = "Association for Computational Linguistics",
    url = "https://aclanthology.org/2021.conll-1.15",
    doi = "10.18653/v1/2021.conll-1.15",
    pages = "194--209",
}

@misc{li2019visualbert,
      title={VisualBERT: A Simple and Performant Baseline for Vision and Language}, 
      author={Liunian Harold Li and Mark Yatskar and Da Yin and Cho-Jui Hsieh and Kai-Wei Chang},
      year={2019},
      eprint={1908.03557},
      archivePrefix={arXiv},
      primaryClass={cs.CV}
}

@inproceedings{calderon-etal-2022-docogen,
    title = "{D}o{C}o{G}en: {D}omain Counterfactual Generation for Low Resource Domain Adaptation",
    author = "Calderon, Nitay  and
      Ben-David, Eyal  and
      Feder, Amir  and
      Reichart, Roi",
    editor = "Muresan, Smaranda  and
      Nakov, Preslav  and
      Villavicencio, Aline",
    booktitle = "Proceedings of the 60th Annual Meeting of the Association for Computational Linguistics (Volume 1: Long Papers)",
    month = may,
    year = "2022",
    address = "Dublin, Ireland",
    publisher = "Association for Computational Linguistics",
    url = "https://aclanthology.org/2022.acl-long.533/",
    doi = "10.18653/v1/2022.acl-long.533",
    pages = "7727--7746"
}

@inproceedings{feng-etal-2021-language,
    title = "Language Model as an Annotator: Exploring {D}ialo{GPT} for Dialogue Summarization",
    author = "Feng, Xiachong  and
      Feng, Xiaocheng  and
      Qin, Libo  and
      Qin, Bing  and
      Liu, Ting",
    booktitle = "Proceedings of the 59th Annual Meeting of the Association for Computational Linguistics and the 11th International Joint Conference on Natural Language Processing (Volume 1: Long Papers)",
    month = aug,
    year = "2021",
    address = "Online",
    publisher = "Association for Computational Linguistics",
    url = "https://aclanthology.org/2021.acl-long.117",
    doi = "10.18653/v1/2021.acl-long.117",
    pages = "1479--1491",
}

@inproceedings{schick-schutze-2021-generating,
    title = "Generating Datasets with Pretrained Language Models",
    author = {Schick, Timo  and
      Sch{\"u}tze, Hinrich},
    booktitle = "Proceedings of the 2021 Conference on Empirical Methods in Natural Language Processing",
    month = nov,
    year = "2021",
    address = "Online and Punta Cana, Dominican Republic",
    publisher = "Association for Computational Linguistics",
    url = "https://aclanthology.org/2021.emnlp-main.555",
    doi = "10.18653/v1/2021.emnlp-main.555",
    pages = "6943--6951",
}

@article{desai2021nice, title={Nice Perfume. How Long Did You Marinate in It? Multimodal Sarcasm Explanation}, volume={36}, url={https://ojs.aaai.org/index.php/AAAI/article/view/21300}, DOI={10.1609/aaai.v36i10.21300}, abstractNote={Sarcasm is a pervading linguistic phenomenon and highly challenging to explain due to its subjectivity, lack of context and deeply-felt opinion. In the multimodal setup, sarcasm is conveyed through the incongruity between the text and visual entities. Although recent approaches deal with sarcasm as a classification problem, it is unclear why an online post is identified as sarcastic. Without proper explanation, end users may not be able to perceive the underlying sense of irony. In this paper, we propose a novel problem -- Multimodal Sarcasm Explanation (MuSE) -- given a multimodal sarcastic post containing an image and a caption, we aim to generate a natural language explanation to reveal the intended sarcasm. To this end, we develop MORE, a new dataset with explanation of 3510 sarcastic multimodal posts. Each explanation is a natural language (English) sentence describing the hidden irony. We benchmark MORE by employing a multimodal Transformer-based architecture. It incorporates a cross-modal attention in the Transformer’s encoder which attends to the distinguishing features between the two modalities. Subsequently, a BART-based auto-regressive decoder is used as the generator. Empirical results demonstrate convincing results over various baselines (adopted for MuSE) across five evaluation metrics. We also conduct human evaluation on predictions and obtain Fleiss’ Kappa score of 0.4 as a fair agreement among 25 evaluators.}, number={10}, journal={Proceedings of the AAAI Conference on Artificial Intelligence}, author={Desai, Poorav and Chakraborty, Tanmoy and Akhtar, Md Shad}, year={2022}, month={Jun.}, pages={10563-10571} }

@inproceedings{lundberg2017unified,
author = {Lundberg, Scott M. and Lee, Su-In},
title = {A unified approach to interpreting model predictions},
year = {2017},
isbn = {9781510860964},
publisher = {Curran Associates Inc.},
address = {Red Hook, NY, USA},
booktitle = {Proceedings of the 31st International Conference on Neural Information Processing Systems},
pages = {4768–4777},
numpages = {10},
location = {Long Beach, California, USA},
series = {NIPS'17}
}

@inproceedings{ribeiro2016whyitrustyou,
author = {Ribeiro, Marco Tulio and Singh, Sameer and Guestrin, Carlos},
title = {"Why Should I Trust You?": Explaining the Predictions of Any Classifier},
year = {2016},
isbn = {9781450342322},
publisher = {Association for Computing Machinery},
address = {New York, NY, USA},
url = {https://doi.org/10.1145/2939672.2939778},
doi = {10.1145/2939672.2939778},
booktitle = {Proceedings of the 22nd ACM SIGKDD International Conference on Knowledge Discovery and Data Mining},
pages = {1135–1144},
numpages = {10},
location = {San Francisco, California, USA},
series = {KDD '16}
}

@InProceedings{kayser2021evil,
    author    = {Kayser, Maxime and Camburu, Oana-Maria and Salewski, Leonard and Emde, Cornelius and Do, Virginie and Akata, Zeynep and Lukasiewicz, Thomas},
    title     = {E-ViL: A Dataset and Benchmark for Natural Language Explanations in Vision-Language Tasks},
    booktitle = {Proceedings of the IEEE/CVF International Conference on Computer Vision (ICCV)},
    month     = {October},
    year      = {2021},
    pages     = {1244-1254}
}

@InProceedings{wang2022ofa,
  title = 	 {{OFA}: Unifying Architectures, Tasks, and Modalities Through a Simple Sequence-to-Sequence Learning Framework},
  author =       {Wang, Peng and Yang, An and Men, Rui and Lin, Junyang and Bai, Shuai and Li, Zhikang and Ma, Jianxin and Zhou, Chang and Zhou, Jingren and Yang, Hongxia},
  booktitle = 	 {Proceedings of the 39th International Conference on Machine Learning},
  pages = 	 {23318--23340},
  year = 	 {2022},
  editor = 	 {Chaudhuri, Kamalika and Jegelka, Stefanie and Song, Le and Szepesvari, Csaba and Niu, Gang and Sabato, Sivan},
  volume = 	 {162},
  series = 	 {Proceedings of Machine Learning Research},
  month = 	 {17--23 Jul},
  publisher =    {PMLR},
  url = 	 {https://proceedings.mlr.press/v162/wang22al.html}
}

@article{Cohen1960Coefficient,
	author = {Cohen, J.},
	journal = {Educational and Psychological Measurement},
	number = {1},
	year = {1960},
	month = {4},
	pages = {37--46},
	publisher = {SAGE Publications},
	title = {A {Coefficient} of {Agreement} for {Nominal} {Scales}},
	volume = {20},
}

@misc{liu2023improved,
      title={Improved Baselines with Visual Instruction Tuning}, 
      author={Haotian Liu and Chunyuan Li and Yuheng Li and Yong Jae Lee},
      year={2023},
      eprint={2310.03744},
      archivePrefix={arXiv},
      primaryClass={cs.CV}
}

@misc{petsiuk2018riserandomizedinputsampling,
      title={RISE: Randomized Input Sampling for Explanation of Black-box Models}, 
      author={Vitali Petsiuk and Abir Das and Kate Saenko},
      year={2018},
      eprint={1806.07421},
      archivePrefix={arXiv},
      primaryClass={cs.CV},
      url={https://arxiv.org/abs/1806.07421}, 
}

@misc{vanderlinden2019globalaggregationslocalexplanations,
      title={Global Aggregations of Local Explanations for Black Box models}, 
      author={Ilse van der Linden and Hinda Haned and Evangelos Kanoulas},
      year={2019},
      eprint={1907.03039},
      archivePrefix={arXiv},
      primaryClass={cs.IR},
      url={https://arxiv.org/abs/1907.03039}, 
}

@misc{smilkov2017smoothgradremovingnoiseadding,
      title={SmoothGrad: removing noise by adding noise}, 
      author={Daniel Smilkov and Nikhil Thorat and Been Kim and Fernanda Viégas and Martin Wattenberg},
      year={2017},
      eprint={1706.03825},
      archivePrefix={arXiv},
      primaryClass={cs.LG},
      url={https://arxiv.org/abs/1706.03825}, 
}

@inproceedings{sundararajan2017axiomaticattributiondeepnetworks,
author = {Sundararajan, Mukund and Taly, Ankur and Yan, Qiqi},
title = {Axiomatic attribution for deep networks},
year = {2017},
publisher = {JMLR.org},
booktitle = {Proceedings of the 34th International Conference on Machine Learning - Volume 70},
pages = {3319–3328},
numpages = {10},
location = {Sydney, NSW, Australia},
series = {ICML'17}
}

@inproceedings{
chang2019explainingimageclassifierscounterfactual,
title={Explaining Image Classifiers by Counterfactual Generation},
author={Chun-Hao Chang and Elliot Creager and Anna Goldenberg and David Duvenaud},
booktitle={International Conference on Learning Representations},
year={2019},
url={https://openreview.net/forum?id=B1MXz20cYQ},
}

@inproceedings{Mothilal_2020, series={FAT* ’20},
   title={Explaining machine learning classifiers through diverse counterfactual explanations},
   url={http://dx.doi.org/10.1145/3351095.3372850},
   DOI={10.1145/3351095.3372850},
   booktitle={Proceedings of the 2020 Conference on Fairness, Accountability, and Transparency},
   publisher={ACM},
   author={Mothilal, Ramaravind K. and Sharma, Amit and Tan, Chenhao},
   year={2020},
   month=jan, collection={FAT* ’20} }

@INPROCEEDINGS {zhou2015learningdeepfeaturesdiscriminative,
author = { Zhou, Bolei and Khosla, Aditya and Lapedriza, Agata and Oliva, Aude and Torralba, Antonio },
booktitle = { 2016 IEEE Conference on Computer Vision and Pattern Recognition (CVPR) },
title = {{ Learning Deep Features for Discriminative Localization }},
year = {2016},
volume = {},
ISSN = {1063-6919},
pages = {2921-2929},
doi = {10.1109/CVPR.2016.319},
url = {https://doi.ieeecomputersociety.org/10.1109/CVPR.2016.319},
publisher = {IEEE Computer Society},
address = {Los Alamitos, CA, USA},
month =Jun}

@ARTICLE{baltrušaitis2017multimodalmachinelearningsurvey,
  author={Baltrušaitis, Tadas and Ahuja, Chaitanya and Morency, Louis-Philippe},
  journal={IEEE Transactions on Pattern Analysis and Machine Intelligence}, 
  title={Multimodal Machine Learning: A Survey and Taxonomy}, 
  year={2019},
  volume={41},
  number={2},
  pages={423-443},
  doi={10.1109/TPAMI.2018.2798607}}

@inproceedings{
turpin2023language,
title={Language Models Don't Always Say What They Think: Unfaithful Explanations in Chain-of-Thought Prompting},
author={Miles Turpin and Julian Michael and Ethan Perez and Samuel R. Bowman},
booktitle={Thirty-seventh Conference on Neural Information Processing Systems},
year={2023},
url={https://openreview.net/forum?id=bzs4uPLXvi}
}

@inproceedings{ijcai2024p684,
  title     = {SEMANTIFY: Unveiling Memes with Robust Interpretability beyond Input Attribution},
  author    = {Bandyopadhyay, Dibyanayan and Ganguly, Asmit and Gain, Baban and Ekbal, Asif},
  booktitle = {Proceedings of the Thirty-Third International Joint Conference on
               Artificial Intelligence, {IJCAI-24}},
  publisher = {International Joint Conferences on Artificial Intelligence Organization},
  editor    = {Kate Larson},
  pages     = {6189--6197},
  year      = {2024},
  month     = {8},
  note      = {Main Track},
  doi       = {10.24963/ijcai.2024/684},
  url       = {https://doi.org/10.24963/ijcai.2024/684},
}

@inproceedings{atanasova-etal-2023-faithfulness,
    title = "Faithfulness Tests for Natural Language Explanations",
    author = "Atanasova, Pepa  and
      Camburu, Oana-Maria  and
      Lioma, Christina  and
      Lukasiewicz, Thomas  and
      Simonsen, Jakob Grue  and
      Augenstein, Isabelle",
    editor = "Rogers, Anna  and
      Boyd-Graber, Jordan  and
      Okazaki, Naoaki",
    booktitle = "Proceedings of the 61st Annual Meeting of the Association for Computational Linguistics (Volume 2: Short Papers)",
    month = jul,
    year = "2023",
    address = "Toronto, Canada",
    publisher = "Association for Computational Linguistics",
    url = "https://aclanthology.org/2023.acl-short.25/",
    doi = "10.18653/v1/2023.acl-short.25",
    pages = "283--294"
}

@inproceedings{siegel2024probabilitiesmatterfaithfulmetric,
    title = "The Probabilities Also Matter: A More Faithful Metric for Faithfulness of Free-Text Explanations in Large Language Models",
    author = "Siegel, Noah  and
      Camburu, Oana-Maria  and
      Heess, Nicolas  and
      Perez-Ortiz, Maria",
    editor = "Ku, Lun-Wei  and
      Martins, Andre  and
      Srikumar, Vivek",
    booktitle = "Proceedings of the 62nd Annual Meeting of the Association for Computational Linguistics (Volume 2: Short Papers)",
    month = aug,
    year = "2024",
    address = "Bangkok, Thailand",
    publisher = "Association for Computational Linguistics",
    url = "https://aclanthology.org/2024.acl-short.49/",
    doi = "10.18653/v1/2024.acl-short.49",
    pages = "530--546"
}

@inproceedings{bandyopadhyay-etal-2024-seeing,
    title = "Seeing Through {V}isual{BERT}: A Causal Adventure on Memetic Landscapes",
    author = "Bandyopadhyay, Dibyanayan  and
      Hasanuzzaman, Mohammed  and
      Ekbal, Asif",
    editor = "Al-Onaizan, Yaser  and
      Bansal, Mohit  and
      Chen, Yun-Nung",
    booktitle = "Findings of the Association for Computational Linguistics: EMNLP 2024",
    month = nov,
    year = "2024",
    address = "Miami, Florida, USA",
    publisher = "Association for Computational Linguistics",
    url = "https://aclanthology.org/2024.findings-emnlp.629/",
    doi = "10.18653/v1/2024.findings-emnlp.629",
    pages = "10715--10731"
}

@inproceedings{chen2024mllmasajudgeassessingmultimodalllmasajudge,
  author={Dongping Chen and Ruoxi Chen and Shilin Zhang and Yaochen Wang and Yinuo Liu and Huichi Zhou and Qihui Zhang and Yao Wan and Pan Zhou and Lichao Sun},
  title={MLLM-as-a-Judge: Assessing Multimodal LLM-as-a-Judge with Vision-Language Benchmark},
  year={2024},
  cdate={1704067200000},
  url={https://openreview.net/forum?id=dbFEFHAD79},
  booktitle={ICML}
}

@article{radford2019language,
  title={Language Models are Unsupervised Multitask Learners},
  author={Radford, Alec and Wu, Jeff and Child, Rewon and Luan, David and Amodei, Dario and Sutskever, Ilya},
  year={2019}
}

@InProceedings{Park_2018_CVPR,
author = {Park, Dong Huk and Hendricks, Lisa Anne and Akata, Zeynep and Rohrbach, Anna and Schiele, Bernt and Darrell, Trevor and Rohrbach, Marcus},
title = {Multimodal Explanations: Justifying Decisions and Pointing to the Evidence},
booktitle = {Proceedings of the IEEE Conference on Computer Vision and Pattern Recognition (CVPR)},
month = {June},
year = {2018}
}

@misc{megahed2025adapting,
    title={Adapting OpenAI's CLIP Model for Few-Shot Image Inspection in Manufacturing Quality Control: An Expository Case Study with Multiple Application Examples},
    author={Fadel M. Megahed and Ying-Ju Chen and Bianca Maria Colosimo and Marco Luigi Giuseppe Grasso and L. Allison Jones-Farmer and Sven Knoth and Hongyue Sun and Inez Zwetsloot},
    year={2025},
    eprint={2501.12596},
    archivePrefix={arXiv},
    primaryClass={cs.CV}
}

@article{JI2025102794,
title = {Pixel-level semantic parsing in complex industrial scenarios using large vision-language models},
journal = {Information Fusion},
volume = {116},
pages = {102794},
year = {2025},
issn = {1566-2535},
doi = {https://doi.org/10.1016/j.inffus.2024.102794},
url = {https://www.sciencedirect.com/science/article/pii/S1566253524005724},
author = {Xiaofeng Ji and Faming Gong and Nuanlai Wang and Yanpu Zhao and Yuhui Ma and Zhuang Shi}
}

@inproceedings{singh2022flava,
  author = {
    Amanpreet Singh and
    Ronghang Hu and
    Vedanuj Goswami and
    Guillaume Couairon and
    Wojciech Galuba and
    Marcus Rohrbach and
    Douwe Kiela
  },
  title = {
    {FLAVA:} {A} Foundational Language And 
    Vision Alignment Model
  },
  booktitle={CVPR},
  year={2022}
}

@misc{bai2023qwenvlversatilevisionlanguagemodel,
      title={Qwen-VL: A Versatile Vision-Language Model for Understanding, Localization, Text Reading, and Beyond}, 
      author={Jinze Bai and Shuai Bai and Shusheng Yang and Shijie Wang and Sinan Tan and Peng Wang and Junyang Lin and Chang Zhou and Jingren Zhou},
      year={2023},
      eprint={2308.12966},
      archivePrefix={arXiv},
      primaryClass={cs.CV},
      url={https://arxiv.org/abs/2308.12966}, 
}
\bibliographystyle{acl_natbib}

\newpage
\appendix
\cleardoublepage

For reference throughout the Appendix, the set of notations are referenced in Table \ref{tab:notations}.

\begin{table*}[!h]
\centering
\resizebox{0.85\textwidth}{!}{
\begin{tabular}{@{}lllll@{}}
\toprule
\textbf{Notations}                                                                                      & \textbf{Meaning}                                                         &  &  &  \\ \midrule
$E$                                                                                                     & Encoder of the multimodal classifier (VB, CLIP+MFB)                      &  &  &  \\
$\mathcal{C}_1$                                                                                                   & The feed forward layer of the multimodal classifier                      &  &  &  \\
$M$                                                                                                     & The multimodal classifier, CAuSE seeks to explain ($C_1 \circ E$)              &  &  &  \\
$\psi$                                                                                                  & The linear layer or MLP that projects the encoder representations into the LM   &  &  &  \\
$\phi$                                                                                                  & The LM (GPT-2) finetuned in our framework.                               &  &  &  \\
$\mathcal{A}$                                                                                                     & The aggregator that takes the LM output and aggregates it into vectors.  &  &  &  \\
$\mathcal{C}_2$                                                                                                   & The feed forward layer of CAuSE, which is identical to $\mathcal{C}_1$             &  &  &  \\
\begin{tabular}[c]{@{}l@{}}$F$ (LLM machinery)\end{tabular} & This is the combination of the MLP, LM and the Aggregator ($A \circ \phi \circ \psi$) &  &  &  \\
Explainer                                                                                               & All the components taken together ($C_2 \circ F$)              &  &  &  \\
CAuSE                                                                                                   & The Explainer along with its losses and the training paradigm.           &  &  &  \\ \bottomrule
\end{tabular}
}
\caption{The revised notations that are used throughout the paper}\label{tab:notations}
\end{table*}

\section{Proofs}\label{proofs}

\setcounter{theorem}{0}
\newtheorem{lem}[theorem]{Lemma}
\begin{lem}\label{theo3} (Simulation)
    Considering the weights of $\mathcal{C}_2$ and $\mathcal{C}_1$ remain the same throughout the training process, using IIT between them ensures $F(c) = (\mathcal{A} \circ \phi \circ \psi)(c) = E(t,v)$, where $c = E(t,v)$. Informally, the LLM machinery ($F = (\mathcal{A} \circ \phi \circ \psi)$) simulates the behavior of $E$.\footnote{$A\circ B$ denotes the composition operation between $A$ and $B$.}
\end{lem}
\begin{proof}

This proof establishes the principle in a \textit{simplified and idealized setting} with a two-neuron linear representation, demonstrating that the IIT training objective drives F to behave as an identity function. Also, note $b = (t,v)$

We first assume that the weights of $\mathcal{C}_1$ and $\mathcal{C}_2$ are $w$ and $w'$, respectively, and they remain fixed throughout training. Hence, $w = w'$. For simplicity, we assume both $\mathcal{C}_1$ and $\mathcal{C}_2$ are simple two-layer fully connected feedforward neural networks.

Under intervention of two identical input neurons in $\mathcal{C}_1$ and $\mathcal{C}_2$, and by the definition of IIT training, we have:

\begin{multline}\label{eqn-theo12}
    \text{softmax}(w_{1i} E(s)_i, w_{2i} E(b)_i) = \\
    \text{softmax}(w_{1i}' F(E(s)_i), w_{2i}' F(E(b)_i)), \\
    \forall i \in \{1, 2\}, \text{ and since } w = w', \\
    \frac{\exp\left(\sum_i w_{1i} E(b)_i\right)}{\exp\left(\sum_i w_{1i} E(b)_i\right) + \exp\left(\sum_i w_{2i} E(s)_i\right)} = \\
    \frac{\exp\left(\sum_i w_{1i} F(E(b))_i\right)}{\exp\left(\sum_i w_{1i} F(E(b))_i\right) + \exp\left(\sum_i w_{2i} F(E(s))_i\right)}, \\
    \forall i \in \{1, 2\}.
\end{multline}

Here, $b$ and $s$ are the base and source inputs, respectively, for $E$. Correspondingly, $E(b)$ and $E(s)$ are the base and source inputs for $F$.

Assume, for contradiction, that $E(b) \neq F(E(b))$. Then there exists some $j$ such that $E(b)_j \neq F(E(b))_j$. Let:

\begin{align}
    E[b] &= [p, q] \nonumber \\
    F[E[b]] &= [\rho_1 p, \rho_2 q]
\end{align}

Let $(s_i, b_j)$ pairs be sampled from the Cartesian product $D_E \times D_E$. We do not explicitly enforce the equality constraint $s = b$ during IIT training. However, since sampling spans all of $D_E \times D_E$, there will naturally be cases where $s = b$.

In such cases, by definition, $E(s) = E(b)$.

So, considering $E(s) = E(b)$, and assuming:

\begin{align}
    w_{11} = \beta, \quad w_{12} = \gamma, \quad w_{21} = \delta, \quad w_{22} = \epsilon \nonumber
\end{align}

we can write Equation~\ref{eqn-theo12} as:

\begin{align}\label{alignp}
    &\beta p(1 - \rho_1) + \gamma q(1 - \rho_2) \nonumber \\
    &= \underbrace{
        \log\left(
            \exp(\beta p + \gamma q) + \exp(\delta p + \epsilon q)
        \right)
    }_{\mathcal{Q}} \nonumber \\
    &\quad - 
    \underbrace{
        \log\left(
            \exp(\beta \rho_1 p + \gamma \rho_2 q) 
            + \exp(\delta \rho_1 p + \epsilon \rho_2 q)
        \right)
    }_{\mathcal{Q}^*}
\end{align}

Similarly, considering intervention on the second set of neurons:

\begin{align}\label{alignq}
    \delta p(1 - \rho_1) + \epsilon q(1 - \rho_2) = \mathcal{Q} - \mathcal{Q}^*
\end{align}

Equations~\ref{alignp} and~\ref{alignq} both hold true if any one of the following conditions is satisfied:
\begin{itemize}
    \item [i)] $\beta = \delta$ and $\gamma = \epsilon$
    \item [ii)] $E(b)$ is a zero vector
    \item [iii)] $\rho_1 = \rho_2 = 1$
\end{itemize}

The first condition is impossible, because even though weights of $\mathcal{C}_1$ and $\mathcal{C}_2$ are the same, the weights of $\mathcal{C}_1$ are fixed, and pairwise equality of weights for $\mathcal{C}_2$ is not enforced during training. The second condition is false because $E(b)$ is fixed for any $b$ and is not a zero vector. Therefore, we must have $\rho_1 = \rho_2 = 1$, which implies:

\[
E(b) = F(E(b))
\]
 \end{proof}

\setcounter{theorem}{0}
\begin{theorem}\label{theo5} (Causal Abstraction)
    If $\mathcal{C}_1$ and $\mathcal{C}_2$ become identical (their weights are equal and they are a causal abstraction
of each other), the explainer becomes a causal abstraction of $M$ and \emph{vice versa}.
\end{theorem}

\begin{proof}

The following proof assumes the idealized identity result of Lemma 1 and the injectivity of the encoder E over the domain of interest. $E$ and $F$ are black boxes and do not share internals. We can only inspect their inputs, outputs, and intermediate representations, and we are allowed to intervene on them. We can also intervene on neurons of $\mathcal{C}_1$ and $\mathcal{C}_2$ as they share the same architecture. To verify whether the explainer is a causal abstraction of $M$, we consider three cases:

We assume there exists a correspondence between inputs of $E$, called $x$ and inputs of $F$, called $z$. For intervention on $x$, $do(x=x')$, we intervene $z$ to $z'$ as $do(z=z')$, but we cannot arbitrarily choose $z'$, as a plausible causal abstraction between $M$ and explainer with respect to the input would require for each high-level intervention 
$do(x=x')$, there is a corresponding low-level intervention $do(z=\alpha(x'))$ whose result agrees with the high-level model. Following the standard definition of causal abstraction, we only align interventions on 
$z$ of the form  $z=\alpha(x)$. Arbitrary 
z-values lie outside the image of 
$\alpha$ and thus have no corresponding high‐level intervention, so they are excluded from the abstraction guarantee. $E$ is a good choice for $\alpha$, because it is already a mapping between $x$ to $z$ and if we choose $\alpha(x) = E(x)$, then $y_2=\mathcal{C}_2(F(\alpha(x))) = \mathcal{C}_2(F(E(x))) = \mathcal{C}_2(E(x)) = \mathcal{C}_1(E(x))=y_1$.

If we intervene $x$ to $x'$, we will also intervene $z$ to $\alpha(x')$. Also,
$y_1^{do(x=x')} := \mathcal{C}_1(E(x'))$
Under this assumption:
$y_2^{do(z=z')} := \mathcal{C}_2(F(z')) = \mathcal{C}_2(F((\underline{do(z'=\alpha(x'))}))) = \mathcal{C}_2(E(x')) = \mathcal{C}_1(E(x'))=y_1^{do(x=x')}$

We restrict our analysis to $z'$ in the image of $E$, i.e. $z' = E(x')$ for some $x'$. Then there exists a partial inverse $\beta:Im(E)\rightarrow X$, $\beta(E(x')) = x'$, or $\beta(z') = x'$ for all $z' \in Im(E)$.

So if we intervene on $z$ to make it $z'$, we assume that intervened $z'$ is also an image of some intervened $x'$, and $x' = \beta(z')$.
Denote, $y_2^{do(z=z')} := \mathcal{C}_2(F(z'))$

Under these assumptions, $y_1^{do(x=x')} := \mathcal{C}_1(E(x')) = \mathcal{C}_1(E(\underline{do(x'=\beta(z')})) =  \mathcal{C}_1(E(\beta(z'))) = \mathcal{C}_1(F(E(\beta(z'))))
\mathcal{C}_1(F(z')) =
\mathcal{C}_2(F(z')) = \mathcal{C}_2(F(z')) = y_2^{do(z=z')}$.

This is easier to do for intermediate representations. Let us denote the intermediate representation of $F$ as $k_F=F(z)$, where $z = E(x)$. We transform $F(E(x))$ to $F(E(x))'$ as $do(k_F = k_F')$. Before this transformation, as $F(E(x)) = E(x)$, we know the intermediate representation for $E$ (i.e. $k_E$) would be equal to $k_F = F(E(x))$ too. So, $k_E = k_F$. 
As we have $k_E = k_F$, hence all interventions on $k_E$ and $k_F$ are interchangeable.

So, $k_E' := k_F'$, if $do(k_F=k_F')$, and $k_F' := k_E'$, if $do(k_E=k_E')$.
So, under $do(k_F = k_F')$,
$y_2^{do(k_F = k_F')} := \mathcal{C}_2(k_F') = \mathcal{C}_1(k_F') = \mathcal{C}_1(k_E') = y_1^{do(k_E = k_E')}$. Similarly,
$y_1^{do(k_E = k_E')} := \mathcal{C}_1(k_E') = \mathcal{C}_2(k_E') = \mathcal{C}_2(k_F') = y_2^{do(k_F = k_F')}$.

If a neuronal intervention is performed inside $\mathcal{C}_1$ or $\mathcal{C}_2$, the outputs will still match, because $\mathcal{C}_1$ and $\mathcal{C}_2$ are trained to be causally equivalent via the IIT loss and are structurally identical, ensured by $R_{\text{match}}$. Note that any input intervention at the level of $\mathcal{C}_1$ and $\mathcal{C}_2$ falls under the previous case for intervention on intermediate representation, and is already covered.

\end{proof}

\section{VLM self-explanation inconsistency vs CAuSE}
\label{app:VLMvCAuSE}

VLM can be a candidate for generating the explanation of an already trained classifier but as empirically shown in Section \ref{baselines}, VLMs suffer from lower F1 score even after finetuning, due to its non access to the representation from $M$. 
Through an example below (Figure \ref{fig:mot}), we also try to explain the VLM failure example under a counterfactual setting while our method CAuSE succeeds.

This figure demonstrates that our model provides more faithful post-hoc explanations under counterfactual scenarios. Unlike LLMs, which often fail to justify their original explanations when the input is minimally changed, our model updates its explanation to reflect the classifier’s altered decision

Part \ding{172} illustrates that the LLM is unfaithful to its own explanation. It claims that modifying a specific aspect of the image would change its classification from offensive to non-offensive. However, even after altering the input as suggested, the model's label remains unchanged. This demonstrates that standard VLMs often fail to remain consistent with their own explanations.
In contrast, our proposed framework CAuSE exhibits robust behavior in the same scenario. When presented with a counterfactual instance \ding{174}—i.e., a minimally edited input that causes the $M$’s output to change (see M’s output in \ding{173} and \ding{174})—finetuned $\phi$ (trained with $\mathcal{L}_{CAuSE}$) updates its explanation accordingly.

\section{Bound for Full Neuron Coverage}
\label{bound}

We aim to find the number of training repetitions, $N$, required to sample and intervene on all $n$ neurons in the network with a high probability, $1-\delta$. During each repetition, we uniformly sample a fraction $p_s = 0.2$ of the neurons at random.

\subsection*{Derivation of the Bound}
The probability that a specific neuron is \emph{not} sampled in a single round is $(1-p_s)$. Since the sampling rounds are independent, the probability that a particular neuron $i$ is not sampled in any of the $N$ repetitions is $P(E_i) = (1-p_s)^N$.

To ensure all neurons are covered, we must bound the probability that at least one neuron is missed. Using the union bound (Boole's inequality), we can express this probability as:
\begin{equation}
    P\left(\bigcup_{i=1}^{n} E_i\right) \le \sum_{i=1}^{n} P(E_i) = n (1 - p_s)^N
\end{equation}
\begin{figure*}[htbp]
    \centering
    \includegraphics[width=\textwidth]{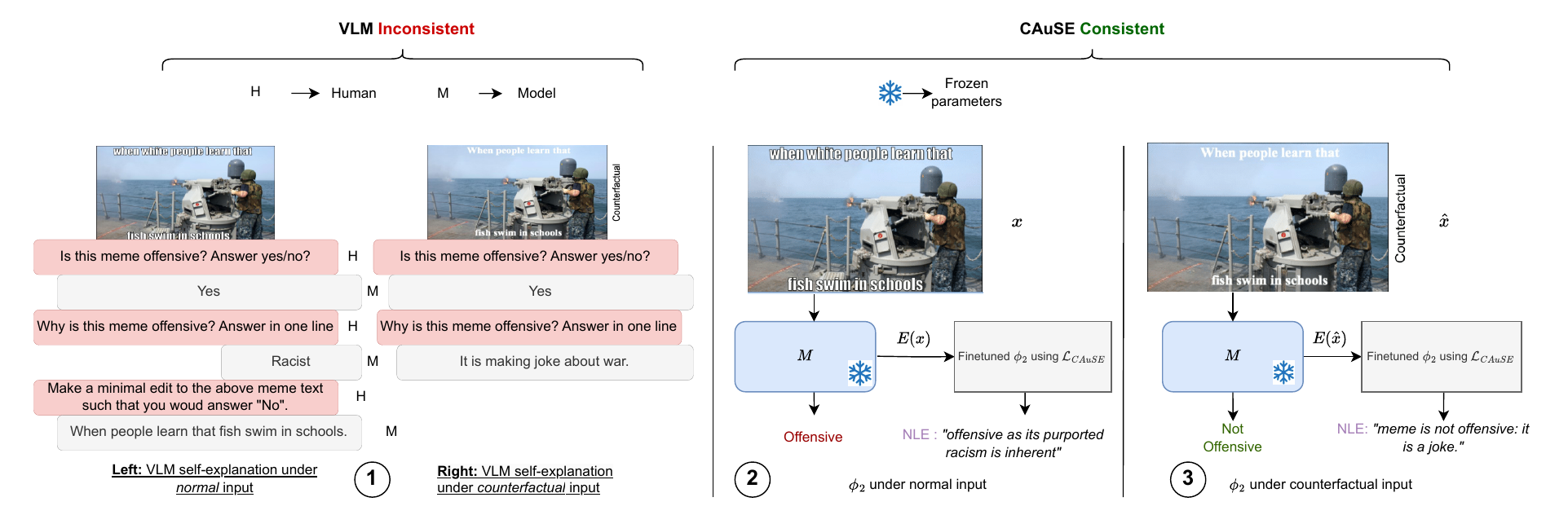}
    \caption{\ding{172} \textit{Left:} Given a meme, a VLM (LLaVA-1.6 \cite{liu2023visualinstructiontuning}) predicts a label and provides a corresponding explanation. \ding{172} \textit{Right:} Under a counterfactual input, the explanation remains unchanged despite the altered input, indicating model unreliability. \ding{173} and \ding{174} illustrate how the NLE adapts to the counterfactual input along with the prediction. The NLE generated by CAuSE is therefore faithful to the multimodal classifier's prediction, as defined by \citet{atanasova-etal-2023-faithfulness}. The example input and its counterfactual along with the VLM input-output pairs in \ding{172} are adopted from \citet{bandyopadhyay-etal-2024-seeing}}

    \label{fig:mot}
\end{figure*}

We require this failure probability to be no greater than a small threshold $\delta$, leading to the inequality $n(1 - p_s)^N \le \delta$. Solving for $N$ by taking the natural logarithm of both sides yields:
\begin{align}
    (1 - p_s)^N &\le \frac{\delta}{n} \\
    N \ln(1 - p_s) &\le \ln(\delta) - \ln(n) \\
    N &\ge \frac{\ln(n) - \ln(\delta)}{-\ln(1 - p_s)} \label{eq:N_bound}
\end{align}
Since $\ln(1 - p_s)$ is negative for $0 < p_s < 1$, the inequality is reversed during the final step of the derivation.

\subsection*{Application to Our Network}
For the classifiers $\mathcal{C}_1$ and $\mathcal{C}_2$, the total number of neurons is $n = 557,440$. Setting a desired success probability of $99.999\%$ (i.e., $\delta = 10^{-5}$), and using $p_s=0.2$, we can apply Equation~\ref{eq:N_bound} to calculate the required number of repetitions:
\begin{equation*}
    N \ge \frac{\ln(557440) - \ln(10^{-5})}{-\ln(0.8)} \approx 110.9
\end{equation*}
Thus, we need at least $N=111$ repetitions. With a batch size of 16, this corresponds to processing $111 \times 16 = 1776$ training samples. As the training sets for both of our datasets contain more than this number of samples, we can be confident that all neurons are sampled and intervened upon within a single training epoch.
\begin{figure*}[t]
    \centering
    \includegraphics[height=3.5cm,keepaspectratio]{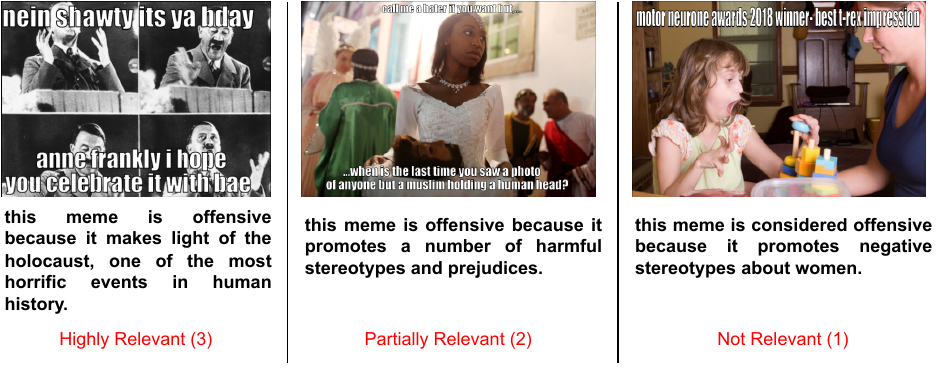}
    \caption{Example of relevance score used for evaluation of generated explanation from LLM.}
    \label{fig:relevance_scoring}
\end{figure*}

\begin{figure*}[htbp]
    \centering
\includegraphics[width=0.75\textwidth]{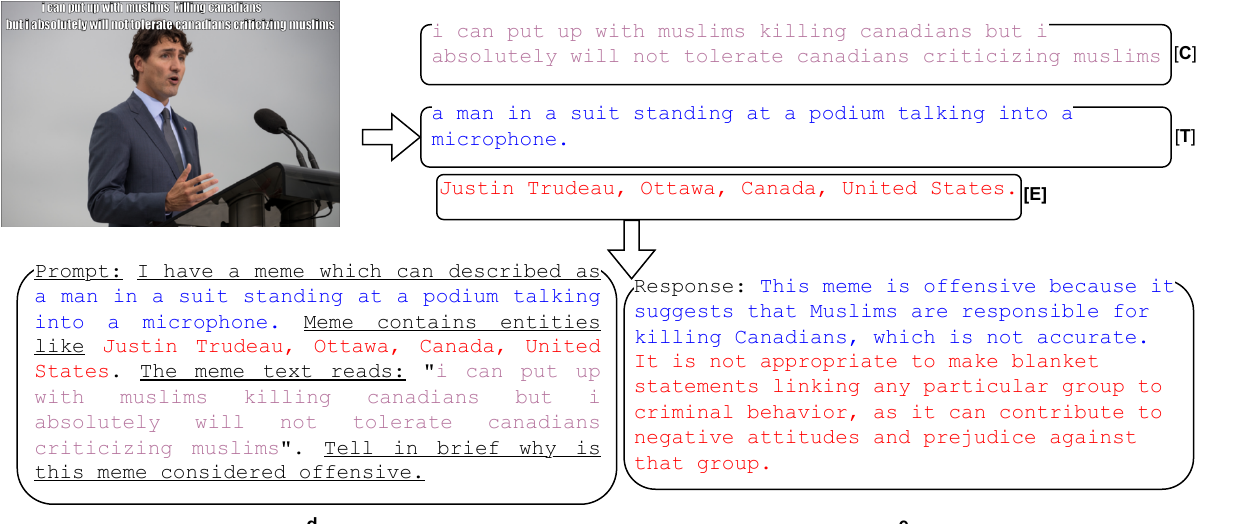}
    \caption{Dataset annotation process via template-based prompt design for LLMs. [T] refers to the caption, [C] refers to the meme text, and [E] is the set of extracted entities. All of them are combined to form the prompt for LLM.  Fixed sections of the prompt that are neither in [C], [T], or [E] are \underline{underlined}. The blue portion from the `Response' field indicated the final response.} 
    
    \label{fig:ds_anno}
\end{figure*}

\section{Hateful Meme Classification Explanation Dataset Preparation} \label{dataprep}

The dataset was created in a similar approach to the methods used in works such as \cite{feng-etal-2021-language,schick-schutze-2021-generating}. Specifically, GPT-4o LLM was used to generate explanations for each meme in the Facebook Hateful meme dataset\cite{kiela2021hateful}.
\begin{table}[!h]
   
\centering
\scriptsize
      \adjustbox{width=\columnwidth}{\begin{tabular}{c|c|c|c|c|c|c}
    \hline
    \multicolumn{3}{c|}{\textbf{Evaluation scores}} &  
    \multicolumn{3}{|c}{\textbf{Relevance Distribution (\%)}} & \multicolumn{1}{|c}{\textbf{Kappa}} \\ \hline
         Model& Relevance & Fluency&NR&PR&HR    & Relevance   \\ \hline
         \textit{GPT-4o}  & 2.34 & 0.94 & 18 & 33 &49  &0.44   \\\hline

        \end{tabular}}
         \caption{Human evaluation results for LLM generated explanation. Here, NR: Not relevant, PR: Partially relevant, HR: Highly relevant.}
        \label{human_eval1}
\end{table}

There were 3 main steps, (i) \textbf{zero-shot template-based prompt design}, (ii) \textbf{inferencing by the LLM}, and (iii) \textbf{manual test set creation using post-processing.} The template-based prompt design incorporates a caption generated by the state-of-the-art captioning model OFA\cite{wang2022ofa}, along with visual entities retrieved from the Google Cloud Vision API\footnote{\url{https://cloud.google.com/vision}}. An example prompt-response pair is shown in Figure \ref{fig:ds_anno}.

In step two, standard LLM greedy inferencing is employed to generate an explanation (\textit{LLM response}) for a given prompt. To ensure the quality of the LLM-generated explanations, three evaluators were hired to check the quality of the explanations for a collection of $200$ sampled memes. The evaluators are chosen from a pool of Masters' CS students in the age range of $22$-$27$, who have previous similar annotation experience. They were asked to rate the \textbf{relevance} of the generated explanation on a scale of \textit{1 (Not relevant), 2 (Partially relevant) or 3 (Highly relevant)}, based on how well it reflected the reason for offensiveness (cf. Figure \ref{fig:relevance_scoring}), and the \textbf{fluency} of the explanation on a \textit{continuous scale from 0 to 1}, based on how grammatically correct it was \cite{desai2021nice,kayser2021evil}.

The complete evaluation statistics are reported in Table \ref{human_eval1}, which shows that the quality of the generated data is quite good. The Cohen kappa \cite{Cohen1960Coefficient} score obtained is $0.44$, indicating fair agreement among the raters.

Finally, from the offensive class examples in the test set, \textbf{every LM explanation was post-processed} \textit{if} the LLM annotation does not properly reflect the ground truth reason behind a meme's purported offensiveness. To achieve this, the three evaluators (the same as the above) were specifically instructed to i) trim a longer explanation into a single sentence, ii) remove redundant information from the LLM responses, and iii) rewrite generated explanations that obtain low relevance scores to make them more relevant to the input meme.

\textbf{Categorization of post-processing.} Out of 596 offensive examples, in the test set, 297 examples had to be manually post-edited (giving the percentage of post-edited samples to 49.8\%) upon manual inspection, which means all of the not relevant (18\%) and almost all of the partially relevant (33\%) samples were manually post-edited. 

\begin{figure*}[t]
    \centering
    \includegraphics[width=\textwidth]{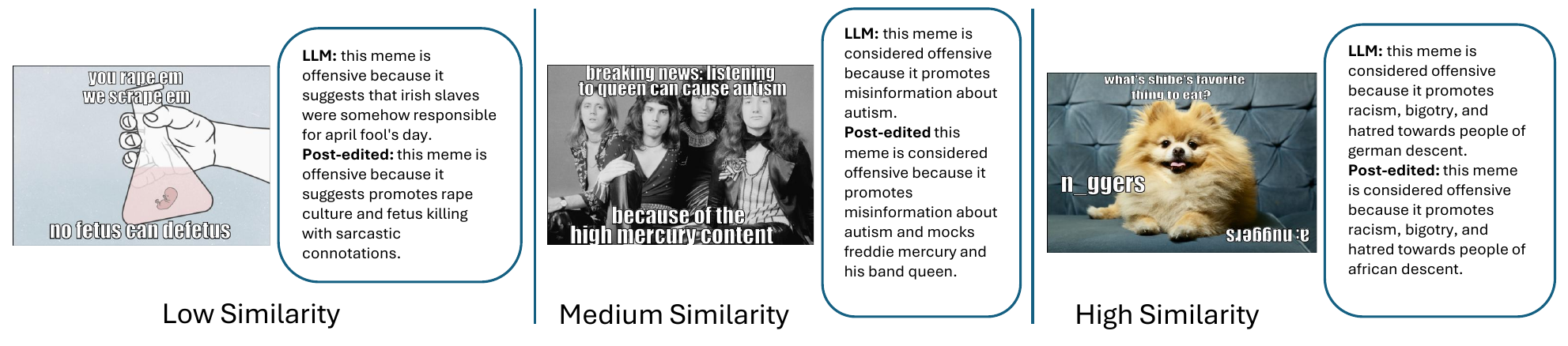}
    \caption{Example post-edits by similarity categorization between before-after edit samples.}
    \label{fig:example-edits}
\end{figure*}

We categorize samples into three groups based on BERTScore similarity between the original and post-edited text: low (below the 25\textsuperscript{th} percentile, $<0.925$), medium (25\textsuperscript{th}-75\textsuperscript{th} percentile, $0.925$--$0.973$), and high (above the 75\textsuperscript{th} percentile, $>0.973$). Notably, the high 25th percentile cutoff of 0.925 indicates that even the low-similarity group required only minor post-editing. Figure \ref{fig:example-edits} shows examples from each group.

\textbf{Annotation portal.} We make use of an annotation portal shown in Figure \ref{fig:anno-portal} for post-processing and generating post-processed gold explanation.

The `Gold explanation' field in the annotation portal refers to the LLM-generated explanation. In Figure \ref{fig:anno-portal}, the LLM-generated explanation does not properly reflect the exact reason behind the meme's offensiveness. This is where the `Edited Reference' field comes into play. This field is kept to make changes to the LLM-generated `Gold Explanation', such that this properly reflects the offensiveness behind a meme. We finally use the edited sentence using the gold-standard annotation as test set used for evaluating all our models.
\begin{figure}[h]
    \centering
    \includegraphics[width=0.9\columnwidth]{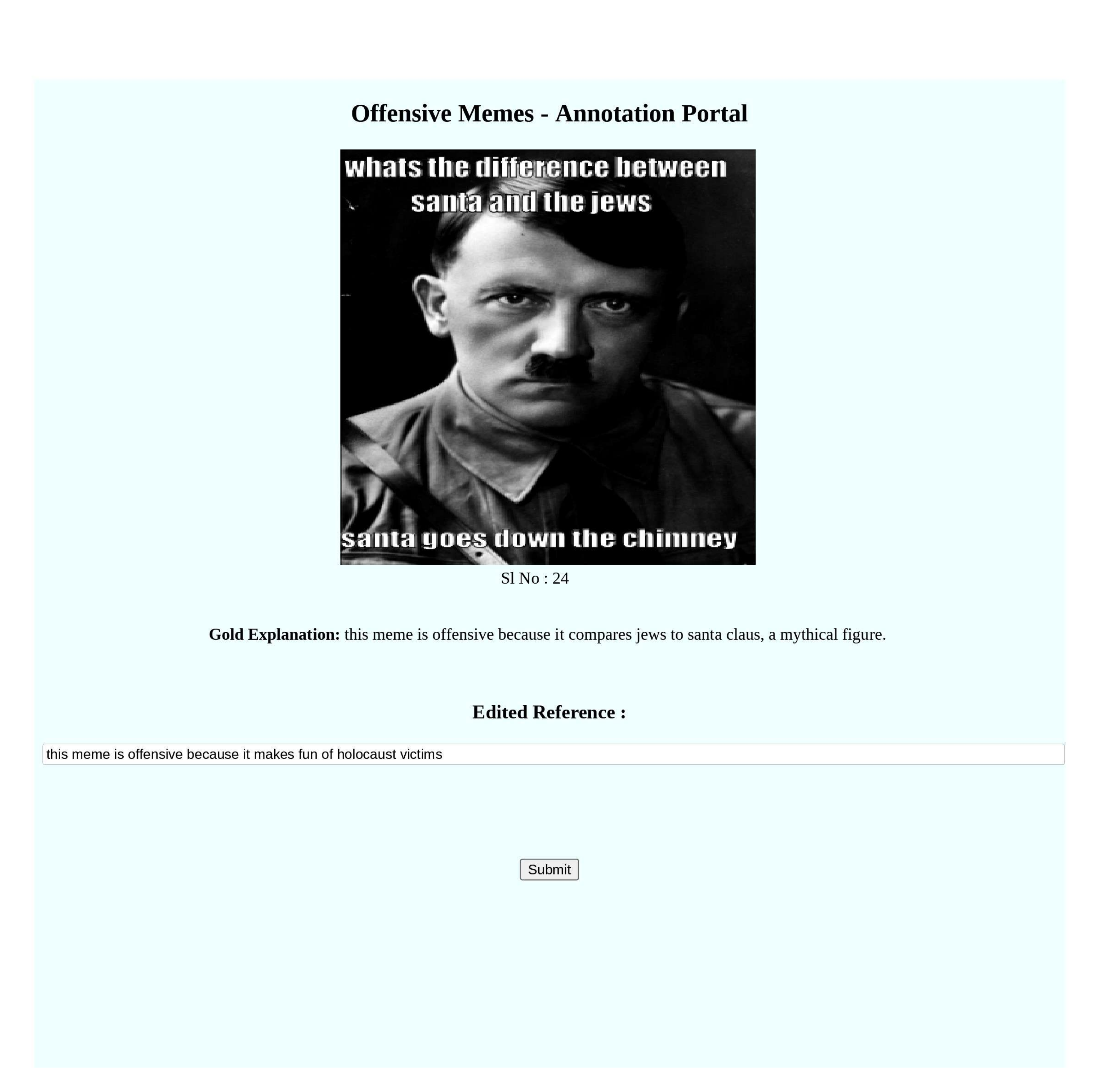}
    \caption{Depiction of the annotation portal.}
    \label{fig:anno-portal}
\end{figure}

\clearpage
\onecolumn

\end{document}